\documentclass[review,onefignum,onetabnum]{siamonline220329}

\usepackage{graphicx}
\usepackage{amsmath}
\usepackage{graphics}
\usepackage{amsfonts}
\usepackage{amssymb}%
\usepackage{amssymb}
\usepackage{mathrsfs}
\usepackage{amsfonts}
\usepackage{amsmath}
\usepackage{bm}
\usepackage{graphicx}
\usepackage{multirow}
\usepackage{caption2}

\usepackage{booktabs}
\usepackage{algorithm}
\usepackage{algorithmic}
\usepackage{subfigure}

\usepackage{float}

\usepackage{lipsum}
\usepackage{amsfonts}
\usepackage{graphicx}
\usepackage{epstopdf}
\usepackage{algorithmic}

\ifpdf
  \DeclareGraphicsExtensions{.eps,.pdf,.png,.jpg}
\else
  \DeclareGraphicsExtensions{.eps}
\fi

\usepackage{enumitem}
\setlist[enumerate]{leftmargin=.5in}
\setlist[itemize]{leftmargin=.5in}

\newsiamremark{remark}{Remark}
\newsiamremark{hypothesis}{Hypothesis}
\crefname{hypothesis}{Hypothesis}{Hypotheses}
\newsiamthm{claim}{Claim}

\headers{Weighted Spectral Filters for Kernel Interpolation}{X.Liu, J. Wang, D. Wang and S. B. Lin}

\title{Weighted Spectral Filters  for  Kernel Interpolation on Spheres: Estimates of Prediction Accuracy for Noisy Data
\thanks{The corresponding author is S. B. Lin (\email{sblin1983@gmail.com}).
		\funding{This work was partially
			supported by the  National Key R\&D Program of China (No.2020YFA0713900) and the
			National Natural Science Foundation of China [Grant Nos. 62276209].  }}}

\author{Xiaotong Liu, Jinxin Wang, Di Wang and Shao-Bo Lin\thanks{X. Liu, J. Wang, D. Wang and S. B. Lin are with the Center for Intelligent Decision-Making and Machine Learning, School of Management, Xi'an Jiaotong University, Xi'an 710049, China.}}

\usepackage{amsopn}
\DeclareMathOperator{\diag}{diag}

\ifpdf
\hypersetup{
	pdftitle={Weighted Spectral Filter for KI},
	pdfauthor={X.Liu, D. Wang and S. B. Lin}
}

\begin{document}
\begin{sloppypar} 
	
\maketitle

\begin{abstract}
Spherical radial-basis-based kernel interpolation abounds in image sciences  including geophysical image reconstruction, climate trends description   and image rendering due to its excellent spatial localization  property and perfect approximation performance. However, in dealing with noisy data, kernel interpolation frequently behaves not so well due to the large condition number of the kernel matrix and instability of the interpolation process. In this paper, we introduce a weighted spectral filter approach to reduce the condition number of the kernel matrix and then stabilize kernel interpolation. The main building blocks of the proposed method are the well developed  spherical positive quadrature rules and  high-pass spectral filters. Using a recently developed integral operator approach for spherical data analysis, we theoretically demonstrate that the proposed  weighted spectral filter approach succeeds in breaking through the bottleneck of kernel interpolation, especially in fitting  noisy data.  We provide optimal approximation rates of the new method to show that our approach does not compromise the predicting accuracy. Furthermore, we conduct both toy simulations and two real-world data experiments with synthetically added noise in geophysical image reconstruction and climate image processing to verify our theoretical assertions and show the feasibility of the weighted spectral filter approach.
\end{abstract}

\begin{keywords}
Spectral filter,  kernel interpolation, approximation, sphere
\end{keywords}

\begin{MSCcodes}
68T05, 94A12, 41A35




\end{MSCcodes}

\section{Introduction}
Spherical data, typically appearing in the form of    circular or directional data,  abound in geophysics \cite{freeden1998constructive}, planetary science \cite{wieczorek1998potential}, meteorology \cite{wang2015absolute},  motion prediction \cite{lang2019numerical},  signal recovery \cite{mcewen2011novel} and image processing \cite{steinke2010nonparametric}.    Spherical data analysis not only enables practitioners to gain deeper insights into the underlying patterns  of the data, but also guides them to make more robust and accurate predictions. For example, in geophysics \cite{freeden1998constructive}, data are collected from satellites,
 and spherical data analysis is carried out  to study the Earth's magnetic field;  in meteorology \cite{wang2015absolute}, data are collected via observing the Earth's atmosphere, and data analysis aims to create models of weather patterns, climate change, and other atmospheric phenomena; in image processing \cite{tsai2010modeling}, data are captured using specialized cameras on a spherical surface, and analyzing spherical data is  to create immersive virtual reality and panoramic images.

In many settings of image processing, including geophysical image reconstruction \cite{alken2021international}, climate change trends  \cite{li1999multiscale} and image rendering \cite{tsai2006all},    data formulated as input-output pairs are collected over the sphere  and  data analysis is conducted  to design efficient algorithms to find an estimator  that provides accurate predictions for new query points. Radial-basis-based kernel interpolation \cite{hubbert2015spherical} (or kernel interpolation for short) is widely used for this purpose, mainly due to its flexibility, accuracy and efficiency.
The flexibility means that kernel interpolation (KI) succeeds in interpolating data on any surface that can be parameterized by a set of spherical coordinates, making it a powerful  tool for a wide range of applications, from modeling the Earth's atmosphere to creating 3D models of planets and other celestial bodies \cite{fornberg2015solving,flyer2014radial,tsai2010modeling}. The accuracy shows that KI  frequently provides accurate predictions when dealing with clean data \cite{narcowich2002scattered,narcowich2007direct,hangelbroek2010kernel,hangelbroek2011kernel}, as it is capable of representing complex functions by using a small number of basis functions. The efficiency illustrates that, for a given kernel,  KI is attained by solving    linear least squares, which is much more user-friendly than   deep learning  \cite{feng2021generalization} and dictionary learning \cite{sun2013dictionary} for spherical data. However, like any other interpolation methods, KI suffers from the instability issue in the sense that its approximation performance changes  drastically with respect to the noise of outputs  and geometric distributions of inputs. A direct consequence    is that KI frequently requires quasi-uniformed sampling \cite{narcowich1998stability,levesley1999norm} and is infeasible  to tackle  spherical data involving  large noise  caused by  measurements \cite{hesse2017radial}, modelling \cite{tsai2006all} or privacy considerations  \cite{nissim2007smooth}.  


\begin{figure}[H]
    \centering
    \includegraphics[scale=0.24]{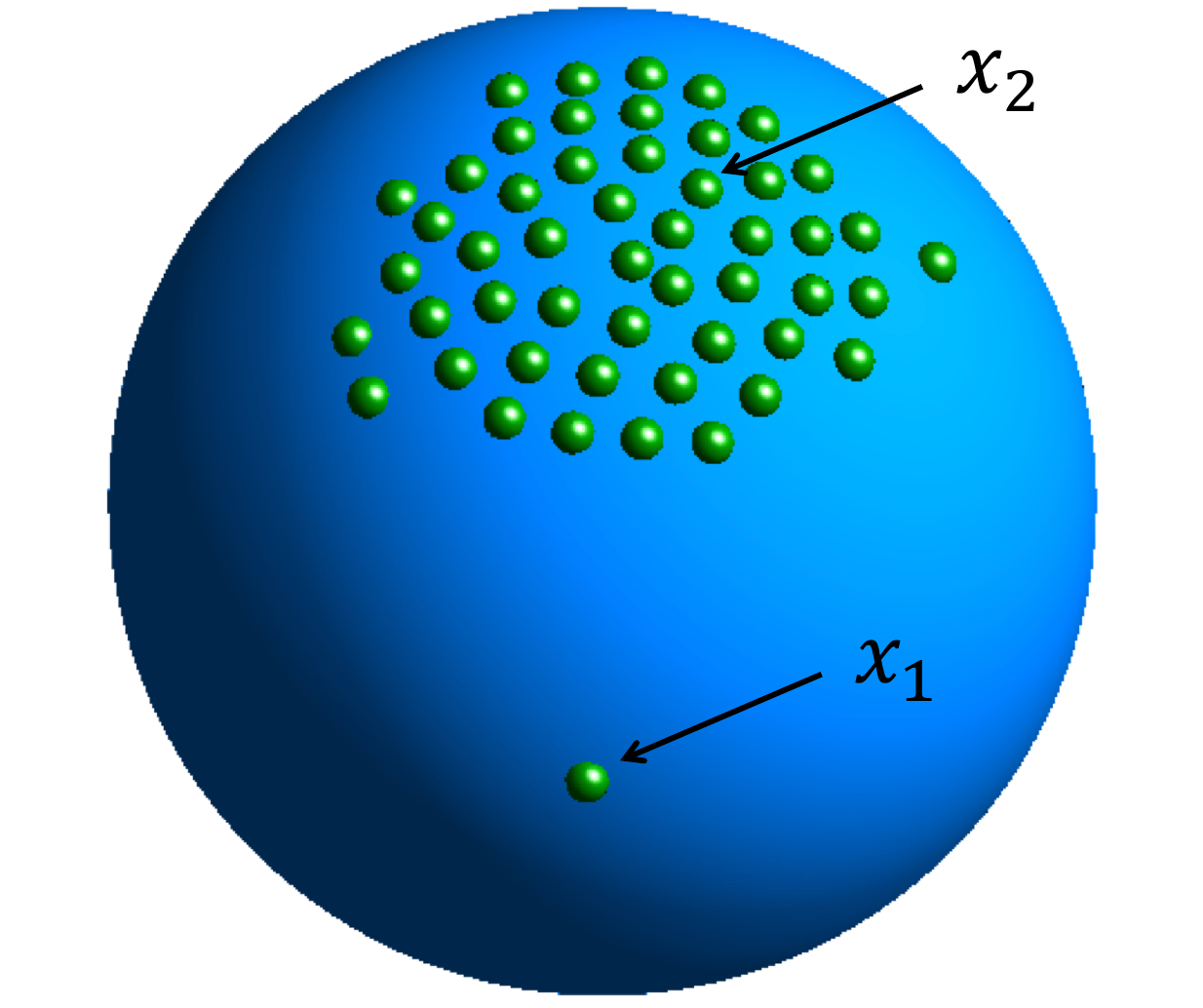}
    \caption{Different roles of the samples  in KI on the sphere}
    \label{fig: x1x2} 
    \vspace{0.1in}
    \end{figure}
\vspace{-0.2in}

Tikhonov regularization \cite{wahba1981spline,hesse2017radial,Feng2021radial} is a preferable approach to conquer the mentioned bottlenecks of KI. However,  theoretical results established in \cite{hesse2017radial} require the noise to be extremely small, while the analysis in \cite{Feng2021radial}
is only suitable to large noise. Furthermore, it is well known \cite{gerfo2008spectral} that Tikhonov regularization suffers from the saturation phenomenon in terms that its approximation rate is saturated on a certain level of the smoothness of the target function and cannot be improved further even when the target function is extremely smooth. 
The aim of this paper  is to propose a   spectral filter approach to enhance the stability of KI without compromising the prediction accuracy. There are roughly three challenges for this purpose, although spectral filter approach has already been  adopted in \cite{gerfo2008spectral,dicker2017kernel,guo2017learning} to tackle  random samples in machine learning. From the methodological perspective, 
deterministic sampling of spherical data places particular emphasis on the locations of samples.  For example, $x_1$ and $x_2$ in  
  \cref{fig: x1x2}  act differently  since $x_2$ can be easily replaced by other points near it while $x_1$ is unique in a large range. From the theoretical consideration,    deterministic sampling of spherical data excludes  statistical analysis tools such as the large number law and concentration inequalities  \cite{dicker2017kernel,guo2017learning}, and thus requires novel analysis tools like the sampling inequality adopted in \cite{hesse2017radial}. From the
  parameter selection side,  as the fitting and stability of the proposed spectral filter methods are determined by  the  filter  parameter,  developing adaptive  parameter selection strategy requires detailed quantitative analysis of the above two terms and novel empirical-to-population techniques, just as concentration inequalities did for random samples.

Our basic idea to settle the above challenges   is to use spherical quadrature rules \cite{mhaskar2001spherical,brown2005approximation}  to mimic concentration inequalities  of   random samples, and then to design  a high-pass filter to exclude  small eigenvalues of the kernel matrix to enhance  the stability of KI without sacrificing its prediction accuracy. In particular, we use   quadrature weights in a spherical quadrature rule  to embody different roles of scattered data  and  utilize some high-pass filters to balance the stability and fitting accuracy in the interpolation process, and then propose a novel fitting scheme called as the weighted spectral filter (WSF) approach. Combining the recently developed integral operator approach for spherical data analysis \cite{Feng2021radial,lin2021subsampling,lin2023dis} and the classical bias-variance analysis for spectral filter algorithms \cite{dicker2017kernel,guo2017learning}, we succeed in deriving an optimal approximation rate of WSF to show that it stabilizes the interpolation process of KI without compromising its prediction accuracy, provided the filter parameter  is appropriately tuned. Furthermore, we provide a novel cross-validation strategy  based on spherical quadrature rules 
  to derive a theoretically optimal filter parameter in the sense that   WSF equipped with the selected parameter  achieves the same optimal approximation rate. Finally, we conduct a series of numerical experiments, including several toy simulations and two real data trials with synthetically added noise, to illustrate the effectiveness of WSF and verify our theoretical assertions. Our main contributions can be concluded as follows:

$\bullet$ Methodology novelty: We propose a type of weighted spectral filters to equip KI to improve its stability. Besides Tikhonov regularization, we find that Landweber iteration and spectral cut-off are also effective and efficient for this purpose. Furthermore, we propose a corresponding cross-validation strategy to select filter parameters for deterministic samples. 

$\bullet$ Theoretical novelty: Large number law and  concentration inequalities are crucial   in analyzing spectral filter algorithm for random samples. Based on spherical positive quadrature rules and integral operator approaches, we develop a similar analysis tool as concentration inequalities for deterministic samples and succeed in deriving optimal approximation rates for WSF. The derived approximation rates adapt to different noise in the sense that they are optimal for different magnitudes  of  the noise. 

$\bullet$ Application guidance: Noise of the data amplifies the limitation of KI.  Practitioners have to turn to other approaches such as deep learning \cite{feng2021generalization} and dictionary  learning \cite{sun2013dictionary} to analyze noisy data on spheres. Our study shows that  adding  high-pass filters with suitable filter parameters to KI is capable of  circumventing the limitation of KI.   Our approaches are verified via both solid theoretical analysis and comprehensive numerical experiments. 

 The rest of the paper is organized as follows. In the next section, we introduce WSF as well as some basic properties of KI. In \cref{Sec:Theoretical Verification}, we provide theoretical verifications of WSF and present the parameter selection strategy.  
In \cref{Sec.Numerical}, we employ both toy simulations and real data trials with synthetically added noise  to show the effectiveness of WSF in fitting noisy data. In \cref{Sec.Proofs}, we prove our main results.

\section{Weighted Spectral Filters for Kernel Interpolation}\label{Sec.Stability}
Let $\mathbb S^d$ be the unit sphere embedded into $\mathbb R^{d+1}$,  the $d+1$-dimensional Euclidean space. 
Assume that the  data  $D=\{(x_i,y_i)\}_{i=1}^{|D|}$ are collected with  $\Lambda=\{x_i\}_{i=1}^{|D|}\subset\mathbb S^d$  and  
\begin{equation}\label{Model1:fixed}
    y_{i}=f^*(x_{i})+\varepsilon_{i},  \qquad\forall\
     i=1,\dots,|D|,
\end{equation}
where $|D|$ denotes the cardinality of $D$, $\{\varepsilon_{i}\}_{i=1}^{|D|}$ is a set of zero-mean i.i.d. random noise satisfying $|\varepsilon_i|\leq M$ for some $M\geq 0$, and $f^*$ is a function to model the relation between the input $x$ and real-valued output $y$.  
The  purpose is to derive an estimator $f_D$ based on  $D$ to approximate  $f^*$ well.   Up till now, there have been numerous  methods proposed for this purpose, typical examples including  spherical polynomials \cite{fasshauer1998scattered,potts2004approximation,filbir2008polynomial}, spherical filtered hyperinterpolation
\cite{sloan2012filtered,lin2021distributed,an2021lasso}, and KI \cite{hubbert2015spherical,hesse2017radial,Feng2021radial}. Due to its 
 flexibility, accuracy and efficiency as discussed in the previous section, we focus on KI in this paper, though we believe that it would be interesting to modify spherical polynomials and spherical filtered hyperinterpolation to fit spherical data whose outputs formed as  \eqref{Model1:fixed}.

\subsection{Kernel interpolation on the sphere}\label{SubSec: KI}
If a function $\phi\in L^2[-1,1]$ satisfies $
\phi(u)=\sum_{k=0}^\infty
\hat{\phi}_k\frac{Z(d,k)}{\Omega_d} P_k^{d+1}(u)
$
with
$$
   \hat{\phi}_k:= \Omega_{d-1} \int_{-1}^1P_k^{d+1}(u)\phi(u)(1-u^2)^{\frac{d-2}2}du\geq0,
$$
then it  is said to be a spherical basis function (SBF) \cite{narcowich2002scattered},
where  $\Omega_d=\frac{2\pi^{\frac{d+1}{2}}}{\Gamma(\frac{d+1}{2})}$ is the volume of $\mathbb S^d$,
$P_k^{d+1}$ is the  Gegenbauer   polynomial of order $\frac{d-1}2$ (also called as the  Legendre polynomial \cite{narcowich2007direct}) and degree $k$ normalized so that
 $P_k^{d+1}(1)=1$ and 
$$
   Z(d,k):=\left\{\begin{array}{ll}
  \frac{2k+d-1}{k+d-1}{\binom{k+d-1}{k}}, & \mbox{if}\ k\geq 1, \\
 1, & \mbox{if}\ k=0 
 \end{array}
 \right.
$$
denotes the dimension of  $\mathbb H_k^d$, the  space of 
spherical harmonics of degree $k$ \cite{muller2006spherical}.  If in addition  $\sum_{k=0}^\infty\hat{\phi}_k {Z(d,k)} <\infty$ and $\hat{\phi}_k>0$   for all $k$, then $\phi$ is said to be  a (strictly) positive definite kernel. It is well known that each SBF  $\phi$ corresponds to a   native space    $\mathcal N_\phi$, defined by \cite{narcowich2002scattered,narcowich2007direct}
$$
\mathcal N_\phi:=\left\{f(x)=\sum_{k=0}^\infty\sum_{\ell=1}^{Z(d,k)}\hat{f}_{k,\ell}Y_{k,\ell}(x): \|f\|_\phi^2:=
\sum_{k=0}^\infty
\hat{\phi}_k^{-1}\sum_{\ell=1}^{Z(d,k)}\hat{f}_{k,\ell}^2
<\infty\right\}, 
$$ 
where $\{Y_{k,\ell}\}_{\ell=1}^{Z(d,k)}$ is  an arbitrary  orthonormal basis of $\mathbb H_k^d$ under the $L^2(\mathbb S^d)$ metric, 
$
\hat{f}_{k,\ell}:=\int_{\mathbb
	S^d}f(x)Y_{k,\ell}(x)d\omega(x)
$
is the Fourier coefficient of $f$, 
and $d\omega$ is  the Lebesgue measure on $\mathbb S^d$ normalized so that $\int_{\mathbb S^d}d\omega(x)=1$.
Recalling \cite{freeden1998constructive,michel2012lectures} that the Sobolev space $H^\gamma$ for $\gamma>0$ is defined as the space of functions in $L^2(\mathbb S^d)$ for which the norm 
$$
      \|f\|_{H^\gamma}:=\left(\sum_{k=0}^\infty(k+1)^{2\gamma}\sum_{\ell=1}^{Z(d,k)}|\hat{f}_{k,\ell}|^2\right)^\frac{1}{2}
$$
is finite. Therefore, for $\gamma>d/2$, it is easy to check that the norm $\|\cdot\|_\phi$ and $\|\cdot\|_{H^\gamma}$ are equivalent when $\hat{\phi}_k\sim(k+1)^{-2\gamma}$,  and then $\mathcal N_\phi$ can be identified with $H^\gamma$.  

KI  (or minimal norm kernel interpolation) defined by
\begin{equation}\label{minimal-norm-interpolation}
f_{D}:=  {\arg\min}_{f\in\mathcal N_\phi}\|f\|_\phi,\qquad
s.t.\quad 
f(x_i)=y_{i},\quad (x_{i},y_{i})\in D,
\end{equation} 
can be analytically solved as  \cite{hubbert2015spherical}
\begin{equation}\label{KI}
f_{D}=\sum_{i=1}^{|D|}a_i\phi_{x_i}, \qquad \mbox{with}\
(a_1,\dots,a_{|D|})^T=:{\bf a}_{D}=\Phi_D^{-1}{\bf y}_D 
\end{equation} 
with  $\Phi_D=\{\phi(x_i\cdot x_j)\}_{i,j=1}^{|D|}$, ${\bf y}_D:=(y_1,\dots,y_{|D|})^T$ and $\phi_{x_i}:=\phi(x_i\cdot)$. 

The approximation performance of $f_D$ depends heavily on the geometric distribution of $\Lambda$, which is measured by the  mesh norm  $
h_{\Lambda}:=\max_{x\in\mathbb S^d}\min_{x_{i}\in \Lambda}\mbox{dist}( x,x_i) 
$ and separation radius $q_{\Lambda}:=\frac12\min_{i\neq i'}\mbox{dist}( x_{i}, x_{i'})$,  where $\mbox{dist}( x,x' )$
is the geodesic distance between   $x$ and $x'$ on $\mathbb S^d$. Denote by $\tau_{\Lambda}:=\frac{h_{\Lambda}}{q_{\Lambda}}$ the mesh ratio which measures the uniformness of $\Lambda$. If there is a $\theta\geq 1$ satisfying $\tau_{\Lambda}\leq \theta$, $\Lambda$ is then said to be $\theta$-quasi-uniform. If $\varepsilon_i=0$ in model \eqref{Model1:fixed} and $f^*\in\mathcal N_\phi$ with $\hat{\phi}_k\sim k^{-2\gamma}$ for $\gamma>d/2$, the approximation error of KI is well  studied 
in the literature \cite{narcowich2002scattered,le2006continuous,narcowich2007direct,hangelbroek2010kernel,hangelbroek2011kernel,hubbert2015spherical}, showing that 
\begin{equation}\label{error-estimate-KI}
\|f_D-f^*\|_{L^2(\mathbb S^d)}\leq ch_\Lambda^{\gamma},
\end{equation}
where  $c$ is a constant depending only on $d$, $\gamma$, and $\|f^*\|_{\phi}$. 
The above estimate verifies the excellent performance of KI in  fitting noise-free data on the sphere.
The problem is, however,  that \eqref{KI} requires to compute the inversion of the kernel matrix $\Phi_D$  whose condition number is generally  large, 
making KI   difficult to fit noisy data.  In fact, 
  Shaback \cite{schaback1995error} established the following interesting inequality:
\begin{equation}\label{uncertainty}
  \sigma_{|D|}(\Phi_D) \leq  \| \phi_x - \mathcal{P}_\Lambda (\phi_x)\|^2_\phi, \qquad x \in \mathbb S^d,
\end{equation}
where $\sigma_{\ell}(U)$ denotes the $\ell$-th  largest   eigenvalue of the matrix $U$  and $\mathcal{P}_\Lambda$ denotes the projector from $\mathcal{N}_\phi$ to its subspace $S_\Lambda:=\mbox{span}
\{\phi_{x_j}\}^{|D|}_{j=1}.$  Recalling \cite{narcowich2002scattered,le2006continuous,narcowich2007direct,hangelbroek2010kernel,hangelbroek2011kernel,hubbert2015spherical} that the power function  $\| \phi_x - \mathcal{P}_\Lambda (\phi_x)\|^2_\phi\stackrel{|D|\rightarrow\infty}{\longrightarrow}0$, we have  $\sigma_{|D|}(\Phi_D)\stackrel{|D|\rightarrow\infty}{\longrightarrow}0$. This demonstrates  that $f_D$ cannot approximate $f^*$ when $|D|$ tends infinity, just as    \cref{fig: drawbackofKI_tdesign} and the following lemma provided in \cite{lin2023dis}  purport to show. 

\vspace{-0.1in}
\begin{figure}[H]
	\centering
 	\subfigure{\includegraphics[scale=0.35]{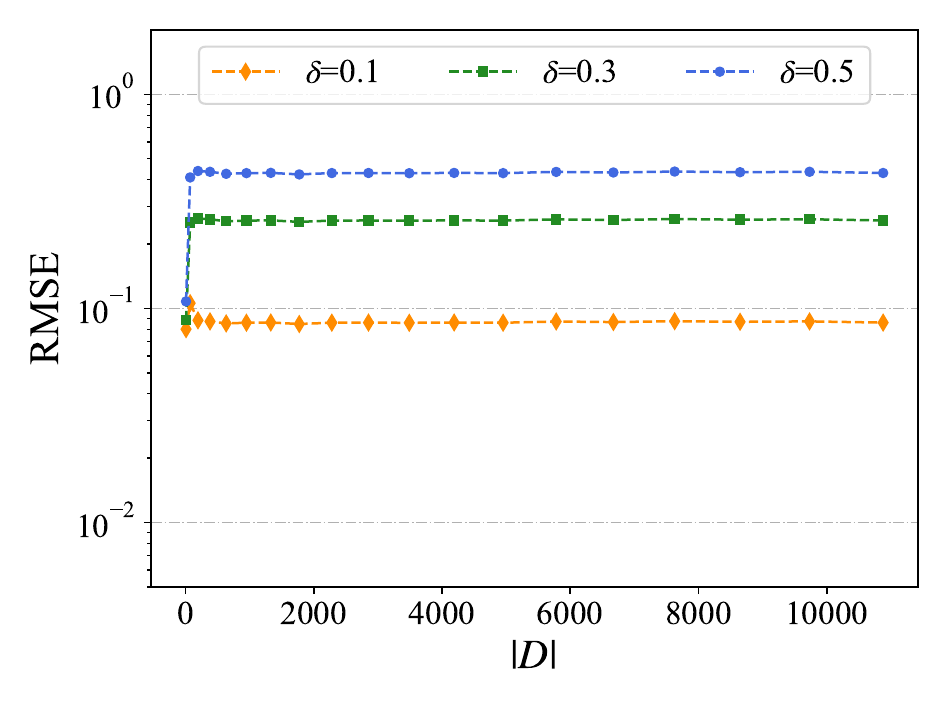}
		\label{subfig: KI_s1_tdesign_RMSE}} 
	\subfigure{\includegraphics[scale=0.37]{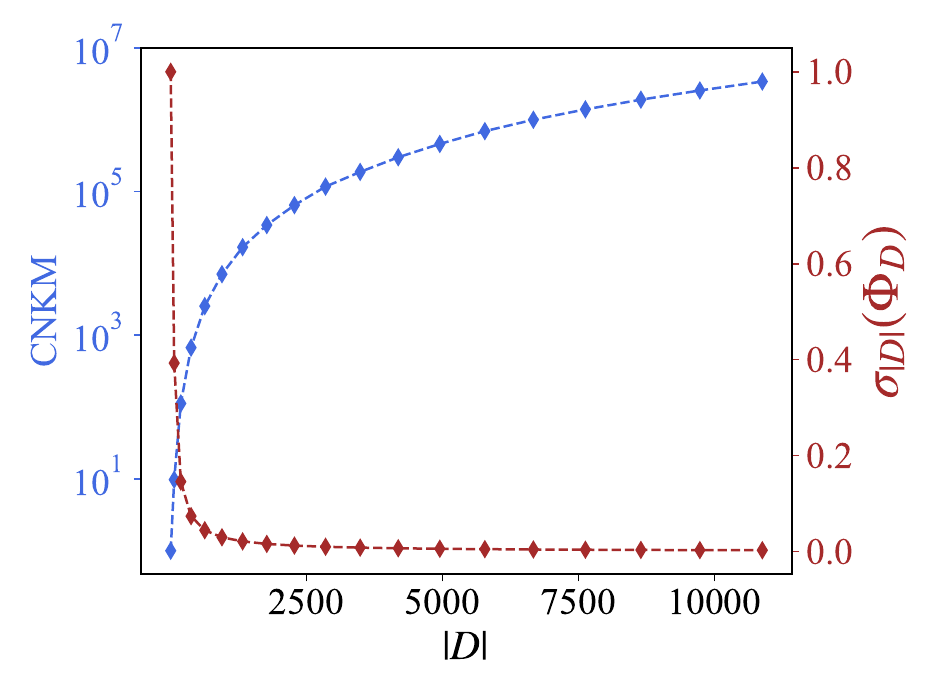}
		\label{subfig: KI_s1_tdesign_CNKM}} 
  \vspace{-0.15in}
	\caption{Relations between training sample size and RMSE (rooted mean square error) of KI/CNKM (condition number of kernel matrix)/$\sigma_{|D|}(\Phi_D)$. The training sample inputs $\{x_i\}_{i=1}^{|D|}$ are generated by Womersley's symmetric spherical $t$-design on $\mathbb S^2$, and the corresponding outputs are generated by the model \eqref{Model1:fixed}, where $f^*$ is defined by \eqref{Model_toysimulation} and $\varepsilon_i$ are independent truncated Gaussian noise $\mathcal N(0,\delta^2)$. We ran simulations 5 times and recorded the average RMSE using the kernel given as in \eqref{Wendland_function}.}
	\label{fig: drawbackofKI_tdesign}
	\vspace{-0.2in}
\end{figure}

\begin{lemma}\label{Lemma:inconsistence}
Let $\Lambda$ be $\theta$-quasi-uniform for some $\theta>1$. Suppose that $\hat \phi_k\sim k^{-2\gamma}$  with $\gamma> d/2$. Let  $\{\varepsilon_i^*\}_{i=1}^{|D|}$ be a set of random variables whose supports are contained in $[-M, -\tau M] \cup [\tau M, M] $, where $0< \tau <1$ is an absolute constant.
Then for all $f \in \mathcal N_\varphi$ with $f|_\Lambda =0,$ there holds almost surely the following inequality:
\begin{equation}
     \|f_{D}-f\|_\phi\geq \tilde{C},
\end{equation}
where $f_D$ is defined by \eqref{minimal-norm-interpolation} with $y_i=f(x_i)+\varepsilon^*_i$ and
$\tilde{C}$ is a positive constant depending only on $\theta, \tau,\gamma,M,$ and $d$.
\end{lemma}

The exhibited phenomenon  
contradicts one's intuition that more data   lead to better prediction \cite{gyorfi2002distribution} and illustrates the  limitation of KI in  fitting noisy data. Furthermore, if the scatter data in $\Lambda$ is not quasi-uniform, i.e., $q_\Lambda\ll h_\Lambda$, then the stability of KI becomes even worse  \cite{narcowich1998stability,levesley1999norm}, just as \cref{fig: drawbackofKI_rotation} purports to show. 

  \vspace{-0.15in}
\begin{figure}[H]
	\centering
	\subfigure{\includegraphics[scale=0.36]{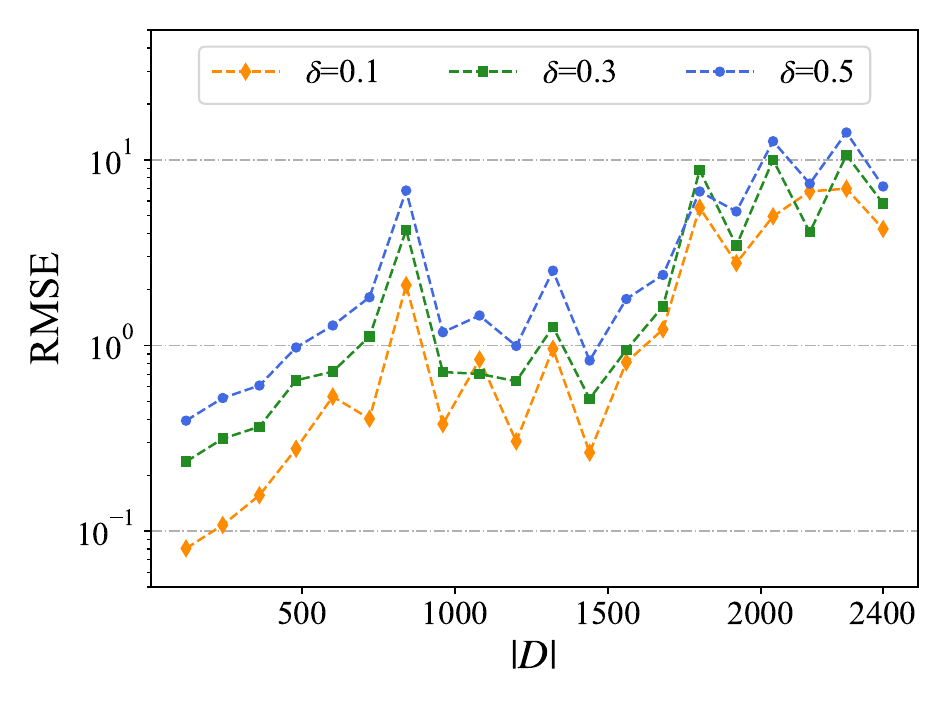}
		\label{subfig: KI_s1_rotation_RMSE}} 
	\subfigure{\includegraphics[scale=0.37]{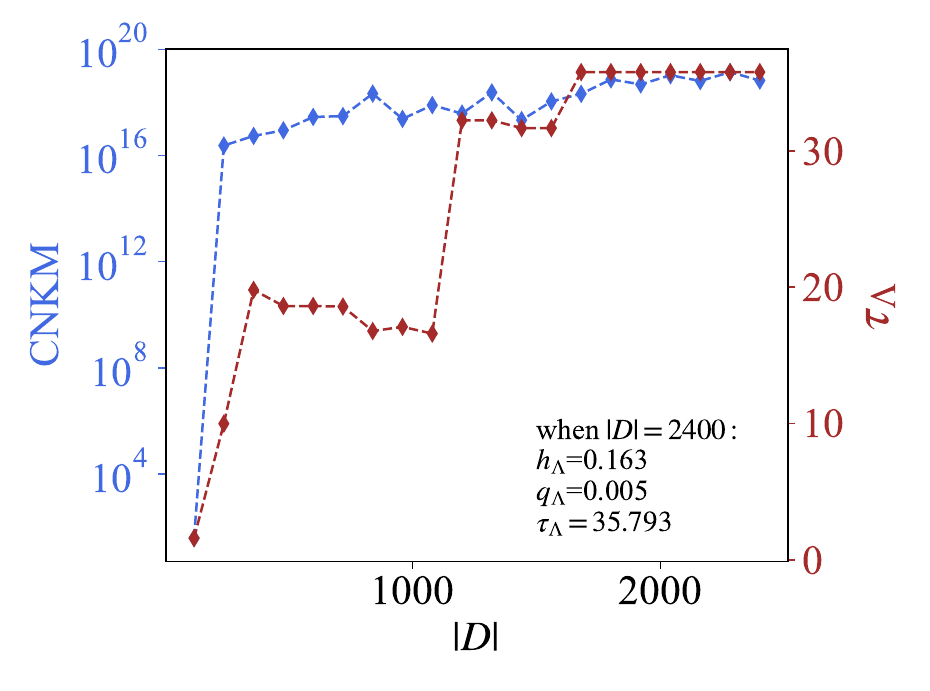}
		\label{subfig: KI_s1_rotation_CNKM}} 
  \vspace{-0.15in}
 	\caption{Relations between training sample size and RMSE of KI/CNKM/$\tau_{\Lambda}$. The training sample inputs $\{x_i\}_{i=1}^{|D|}$ are generated by rotating $t$-designs  on $\mathbb S^2$ whose details can be found in \cref{Sec.Numerical}, and the corresponding outputs are generated by the same way as in  \cref{fig: drawbackofKI_tdesign}. We ran simulations 5 times and recorded the average RMSE.}
	\label{fig: drawbackofKI_rotation}
\end{figure}


\subsection{Weighted spectral filters for KI}
To circumvent the limitation of KI,  we aim to introduce a novel weighted spectral filter (WSF) approach to stabilize the interpolation process. There are two main building blocks of  WSF: weighting and filtering.
Let us first introduce the weighting scheme. For $s\in\mathbb N$,  $\mathcal B_{\Lambda,s}:=\{(w_{i,s},  x_{i}): w_{i,s}> 0
\hbox{~and~}   x_{i}\in \Lambda\}$ is said to be  a positive spherical quadrature rule  of order $s$  \cite{mhaskar2001spherical,brown2005approximation}, if
\begin{equation}\label{eq:quadrature}
\int_{\mathbb S^d}P(x)d\omega(x)=\sum_{x_{i}\in\Lambda} w_{i,s} P(  x_{i}), \qquad \forall P\in \Pi_{s}^d,
\end{equation}
where $\Pi_s^d$ denotes the set of spherical polynomials of degree at most $s$. The existence and construction of the positive spherical quadrature rules   have been extensively studied in the literature \cite{mhaskar2001spherical,brown2005approximation,le2008localized}. 
The following
lemma derived in \cite{mhaskar2001spherical,brown2005approximation}   presents a relation between $s$ and $h_{\Lambda}$ to permit a spherical positive quadrature formula. 

\begin{lemma}\label{Lemma:fixed quadrature}
	If $\Lambda$ is the set of scattered data on the sphere with mesh norm $h_{\Lambda}$, then for any $s_0\leq s\leq c'_1h_\Lambda^{-1}$
	there exists a positive quadrature  rule $\mathcal B_{\Lambda,s}=\{(w_{i,s},  x_{i}): w_{i,s}> 0
	\hbox{~and~}   x_i\in \Lambda\}$ with $\sum_{i=1}^{|\Lambda|}w_{i,s}\leq 1$. If in addition, $\Lambda$ is $\theta$-quasi-uniform, then $w_{i,s}\leq c_2'|\Lambda|^{-1}$. Here, $s_0$ and $c'_1$ are constants depending only on $d$ and $c_2'$ is a constant depending only on $d$ and $\theta$.  
\end{lemma}

Let $W_{D,s}$ be the diagonal matrix with diagonal elements $\{w_{1,s},\dots,w_{|D|,s}\}$.  
It is easy to check that the coefficient vector ${\bf a}_D$ defined in \eqref{KI} satisfies 
\begin{equation}\label{WKI}
       {\bf a}_{D}=W_{D,s}^{1/2}(W_{D,s}^{1/2}\Phi_DW_{D,s}^{1/2})^{-1}W_{D,s}^{1/2} {\bf y}_D, 
\end{equation}
showing that involving $W_{D,s}$ in the matrix-inversion does not affect the approximation performance of KI.  The weighting scheme focuses on replacing $\Phi_D$ in \eqref{KI} by  
\begin{equation}\label{def.psi}
    \Psi_{D,s}:=W_{D,s}^{1/2}\Phi_DW_{D,s}^{1/2}
\end{equation}  
in \eqref{WKI} to embody the different roles of scattered data. The reason for adopting the quadrature weights is to use quadrature formulas \eqref{eq:quadrature} to mimic the concentration inequalities \cite{dicker2017kernel} for theoretical analysis.   

We then introduce the filtering scheme, which is developed  to  approximate    $\Psi_{D,s}^{-1}$ in a stable way. Spectral filter approach \cite{gerfo2008spectral} that designs a specific  high-pass filter on the spectrum of the kernel matrix has been widely used in machine learning for random samples. Though it is  highly non-trivial in theoretical analysis to borrow the idea of spectral approach from \cite{gerfo2008spectral} directly to the noisy scattered data fitting model \eqref{Model1:fixed}, the numerical implementation is relatively easy.   The following definition of  spectral filters can be found in \cite{gerfo2008spectral,guo2017learning}.
\begin{definition}\label{Def:high-pass} 
Let $\kappa$ be the maximal eigenvalue of the matrix $\Psi_{D,s}$.
If for any $\lambda \in (0, \kappa]$, $g_\lambda:[0,\kappa]\rightarrow\mathbb R$ satisfies 
\begin{equation}\label{condition1}
\sup_{0<\sigma\le\kappa}|g_\lambda(\sigma)|\le \frac{b}{\lambda},
\qquad \sup_{0<\sigma\le \kappa}|g_\lambda(\sigma)\sigma|\le b,
\end{equation}
and
\begin{equation}\label{condition2}
\sup_{0<\sigma\le \kappa}|1-g_\lambda(\sigma)\sigma|\sigma^v\le
\tilde{C}_v\lambda^v, \qquad \forall\ 0\leq v\leq
{\nu}_g, 
\end{equation}
with  $b>0$ an absolute constant, $\nu_g>0$, and $\tilde{C}_v>0$ a constant depending only on
$v\in[0,\nu_g],$
$g_\lambda$ is then said to be a high-pass filter for KI with qualification $\nu_g>0$.
\end{definition}

Definition \ref{Def:high-pass} shows that the filter $g_\lambda$, compared with the classical inversion operator $g(t)=\frac1t$, excludes small values  and thus   performs as a high-pass filter. In particular, \eqref{condition1} and \eqref{condition2} quantify the high-pass filter property  of $g_\lambda$ via tuning the filter parameter $\lambda$. 
With the help of the weighting and filtering schemes, we are in a position to present WSF in Algorithm \ref{alg:WSF}.

\begin{algorithm}[t]
\begin{algorithmic}\caption{WSF}
\label{alg:WSF}
\STATE {\bf Inputs}: $D=\{(x_i,y_i)\}_{i=1}^{|D|}$ with $\Lambda=\{x_i\}_{i=1}^{|D|}\subset\mathbb S^d$,  filter $g_\lambda$ with filter parameter $\lambda \in (0, \kappa]$, and kernel $\phi$.

\STATE {\bf Weighting}: Given the set of scattered data $\Lambda$, determine a spherical positive quadrature rule  
\begin{equation}\label{quadrature-rule}
    \mathcal B_{\Lambda,s}=\{(w_{i,s},  x_{i}): w_{i,s}> 0
\hbox{~and~}   x_{i}\in \Lambda\}
\end{equation} 
for $s\in\mathbb N$ as large as possible.

\STATE {\bf Filtering}: Generate the kernel matrix $\Phi_D$ and the weighted kernel matrix $\Psi_{D,s}$. Define 
\begin{equation}\label{spectral-algorithm}
    f_{D,\lambda}=\sum_{i=1}^{|D|}a_i\phi_{x_i} \qquad \mbox{with}\ (a_1,\dots,a_{|D|})^T:=W_{D,s}^{1/2}g_\lambda(\Psi_{D,s})W_{D,s}^{1/2}{\bf y}_D,
\end{equation}
where 
  $g_\lambda(\Psi_{D,s})$ is defined by spectral calculus.

\STATE {\bf Output}: $f_{D,\lambda}$.    
\end{algorithmic}
\end{algorithm}

In computing the quadrature rule in \eqref{quadrature-rule}, it is suggested to use the approach developed in \cite{le2008localized}. As shown in  \cref{Lemma:fixed quadrature}, the degree of quadrature rule depends only $h_\Lambda$, i.e., $s\sim h_\Lambda^{-1}$.   
Though \eqref{spectral-algorithm} presents a uniform framework of  WSF, how to choose the high-pass filter should be specified.  We then present three typical realizations of WSF in the following.
 
\begin{itemize}
	\item \textbf{Tikhonov regularization}: Given a regularization parameter $\mu>0$, define $g_\mu(v)=\frac1{v+\mu}$. The Tikhonov regularization is  defined by 
	\begin{equation}\label{Tikhnov-regularization}
	f_{D,\mu}=\sum_{i=1}^{|D|}a_i\phi_{x_i} \qquad \mbox{with}\
	(a_1,\dots,a_{|D|})^T=W_{D,s}^{1/2}(\Psi_{D,s}+\mu I)^{-1}W_{D,s}^{1/2}{\bf y}_D.
	\end{equation}
	\item \textbf{Landweber iteration}: For a positive definite matrix $A$, if $\lim_{u\rightarrow\infty} A^u=0$ and $I-A$ is strictly positive definite, then the Neumann series implies
	$  
	(I-A)^{-1}=\sum_{k=0}^\infty A^k.
	$ 
	Let  $\tau$ be an arbitrary positive number satisfying $0<\tau\leq \frac1{\kappa}.$  Therefore, we have
	$
	(\Psi_{D,s})^{-1}=\tau\sum_{k=0}^\infty(I-\tau\Psi_{D,s})^k.
	$
	Given a parameter  $l\in \mathbb N$, define the filter  $g_{1/l}(v) =\sum_{k=0}^l(1-v)^k$. The Landweber iteration is given by
	\begin{equation}\label{Landweber}
	f_{D,l}=\sum_{i=1}^{|D|}a_i\phi_{x_i} \qquad \mbox{with}\
	(a_1,\dots,a_{|D|})^T=\tau\sum_{k=0}^lW_{D,s}^{1/2}(I-\tau\Psi_{D,s})^k W_{D,s}^{1/2}{\bf y}_D.
	\end{equation}
	
	\item \textbf{Spectral cut-off}:  
  Perform SVD of the kernel matrix $\Psi_{D,s}=Q\Sigma Q^T$, where 
 $\Sigma=\diag(\sigma_1,\dots,\sigma_{|D|})$ with $\sigma_1\geq \sigma_2\geq\dots\geq\sigma_{|D|}\geq 0$ being a non-increasing array of the eigenvalues of $\Psi_{D,s}$. 
  Given a parameter $\nu>0$, write $\Sigma_\nu=\diag(\sigma_1,\dots,\sigma_i,0,\dots,0)$   where $\sigma_{j}<\nu$ for $j>i$ and $\sigma_j\geq \nu$ for $j\leq i$.  
	Define the filter $g_\nu(v)=\left\{\begin{array}{cc}
	      \frac{1}{v} & \mbox{if}\quad v\geq\nu,\\
 	   0  & \mbox{otherwise}.
	\end{array}\right.$
We then obtain the spectral cut-off estimate by
	\begin{equation}\label{cut-off}
	f_{D,\nu}=\sum_{i=1}^{|D|}a_i\phi_{x_i} \qquad \mbox{with}\
	(a_1,\dots,a_{|D|})^T= W_{D,s}^{1/2}Q\Sigma_\nu^{-1} Q^T W_{D,s}^{1/2}{\bf y}_D.
	\end{equation} 
\end{itemize}

It is obvious that setting  $\mu=0$, $\nu=0$ and $l\rightarrow\infty$ makes \eqref{Tikhnov-regularization}, \eqref{cut-off} and \eqref{Landweber} boil down to KI \eqref{KI}. The filters of Tikhonov regularization, Landweber iteration and spectral cut-off satisfy \eqref{condition1} and \eqref{condition2} for suitable $\nu_g$. In particular, $\nu_g=\infty$ for Landweber iteration and spectral cut-off while $\nu_g=1$ for Tikhonov regularization. Besides these three examples, there are several other filters  corresponding to different learning schemes \cite{gerfo2008spectral} such as the well known $\nu$-method  and iterative Tikhonov. 


\subsection{Parameter selection}
As WSF requires a suitable filter parameter $\lambda$, it is highly desired  to design   feasible and efficient  strategies to adaptively select $\lambda$. In this subsection, we focus on modifying the popular cross-validation in machine learning \cite{gyorfi2002distribution,caponnetto2010cross} for random samples  to be a   weighted cross-validation to determine $\lambda$, described  
 as follows:

$\diamond$ {\it Step 1: Data division.} Divide the data set $D$ into two data subsets: $D^{tr}=\{(x_i^{tr},y_i^{tr})\}$ (the training set) and $D^{val}=\{(x_i^{val},y_i^{val})\}$ (the validation set).

$\diamond$ {\it Step 2: Training.} Given a set of parameters $\Xi_L:=\{\lambda_\ell\}_{\ell=1}^L$, run  WSF \eqref{spectral-algorithm} on $D^{tr}$ with $\lambda=\lambda_1,\dots,\lambda_L$  and obtain a set of estimators $\{f_{D^{tr},\lambda_\ell}\}_{\ell=1}^L$.

$\diamond$ {\it Step 3: Computing quadrature rule on the validation set.} 
Determine a spherical positive quadrature rule
$$
        \mathcal B_{\Lambda^{val},s}=\left\{(w^{val}_{i,s},x_i^{val}):w_{i,s}^{val}>0\ \mbox{and}\  x_i^{val}\in\Lambda^{val}\right\}.
$$

$\diamond$ {\it Step 4: Parameter selection.} Select the   parameter $\hat{\lambda}$ via the validation set $D^{val}$ by
\begin{equation}\label{CV-for-parameter}
\hat{\lambda}={\arg\min}_{1\leq\ell\leq L} \sum_{i=1}
	^{|D^{val}|} w_{i,s}^{val}\left(f_{{D^{tr}},\lambda_\ell}(x_i^{val})-y_i^{val}\right)^2.
\end{equation}

It should be mentioned that there are roughly two differences  between the proposed weighted cross-validation and the classical one \cite{caponnetto2010cross} for random samples. The first one is that we remove the truncation operator in    cross-validation for random samples \cite{caponnetto2010cross}.
Generally speaking,
the truncation step aims to  project an  output function $f: \mathbb S^d \to \mathbb R$ onto
the interval $[-M_y, M_y]$ with $M_y=\max_{i=1,\dots,|D|}{|y_i|}$ by utilizing the projection operator  
$
\pi_{M} f(x)=\mbox{sgn}{f(x)}\min\{M,|f(x)|\}. 
$
For any  $|f^*(x)|\leq M_y, \forall x\in\mathbb S^d$, it is obvious 
$$
\|\pi_{M_y}f-f^*\|_{L^2(\mathbb S^d)}\leq \|f-f^*\|_{L^2(\mathbb S^d)}, \qquad\forall f\in L^2(\mathbb S^d).
$$
The problem is, however, $\pi_{M_y} f$ generally does not belong to $\mathcal N_\phi$, making a contradiction in the sense that one wants to find an estimator in $\mathcal N_\phi$ while the truncation operator  kicks the final estimator out. The other difference, as the key novelty in the proposed method, is to select an optimal parameter via weighted least squares over the validation set like \eqref{CV-for-parameter}. Since we are concerned with deterministic samples,   new tools like spherical positive quadrature rules 
 are needed to build the relation between the (weighted) empirical risk and its population.



In this way, we obtain an  adaptive version of WSF, called as Ada-WSF, $f_{D^{tr},\hat{\lambda}}$ to successfully settle the parameter selection problem. It should be highlighted that different from priori parameter selection strategies presented in \cite[Theorem 4.1]{hesse2017radial}, our proposed weighted cross-validation is  posterior and can be implemented without accessing any a-priori information for the smoothness of  target functions $f^*$ and distributions of   noise $\varepsilon$.

\section{Theoretical Verifications}\label{Sec:Theoretical Verification} In this section, we aim at deriving optimal approximation rates of WSF and compare our results with some related work.

\subsection{Approximation rate analysis}
Before presenting the approximation rates, we first introduce another two SBFs, $\psi$ and $\varphi$, satisfying 
\begin{eqnarray}\label{kernel-relation} 
	\hat{\varphi}_k = \hat{\phi}_k^{\alpha},\quad 
        \hat{\psi}_k = 
	\hat{\phi}_k^{\beta},\qquad  0\leq \beta \leq 1, \alpha\geq 1.
\end{eqnarray} 
Throughout this section, we assume $f^*\in\mathcal N_\varphi$ and the error analysis  is carried out in the metric of $\mathcal N_\psi$.  To demonstrate the power of  spectral filters, we need the following definition of stability error, fitting error and approximation error of WSF.

\begin{definition}\label{def.stability}
Let
\begin{equation}\label{spectral-noise-free}
f^\diamond_{D,\lambda}=\sum_{i=1}^{|D|}a_i\phi_{x_i} \qquad \mbox{with}\ (a_1,\dots,a_{|D|})^T=W^{1/2}_{D,s}g_\lambda(\Psi_{D,s})W^{1/2}_{D,s}{\bf f}^* 
\end{equation}
 be the noise-free version of $f_{D,\lambda}$, where  ${\bf f}^*=(f^*(x_1),\dots,f^*(x_{|D|}))^T$. The stability error, defined by $\|f^\diamond_{D,\lambda}-f_{D,\lambda}\|_\psi$, measures the role of noise in WSF. The fitting error,
defined by $\|f^\diamond_{D,\lambda}-f^*\|_\psi$, quantifies the fitting performance of WSF for tackling noise-free data. The approximation error, defined by $\|f_{D,\lambda}-f^*\|_\psi$, reflects the approximation performance of WSF in handling the noisy data. 
\end{definition}

It follows from the  triangle inequality that
\begin{equation}\label{error-dec-1-spectral}
\|f_{D,\lambda}-f^*\|_\psi\leq
\overbrace{\|f_{D,\lambda}- f^\diamond_{D,\lambda}\|_\psi}^{\mbox{Stability}}+\overbrace{\| f^\diamond_{D,\lambda}-f^*\|_\psi}^{\mbox{Fitting}},
\end{equation}
showing that the approximation error of WSF can be bounded by the sum of the stability error and fitting error. In the following theorem, we  present theoretical guarantees on WSF in demonstrating its stability and fitting errors, respectively. 

\begin{theorem}\label{Theorem:rate-spectral}
	Let $\tilde{\delta}\in(0,1)$, $f_{D,\lambda}$ be defined by \eqref{spectral-algorithm},
 $\hat{\phi}_k\sim k^{-2\gamma}$ with $2\gamma>d$ for any $k\in\mathbb N$ and     $\psi$, $\varphi$ satisfy 
	\eqref{kernel-relation} 
	with $0\leq \beta\leq 1$ and $\alpha\geq 1$.
If $g_\lambda$ satisfies \eqref{condition1} and \eqref{condition2} with qualification $\nu_g>0$,
	$f^*\in \mathcal N_\varphi$, $\{\varepsilon_i\}_{i=1}^{|D|}$ are zero-mean i.i.d. random values satisfying $|\varepsilon_i|\leq M$ for some $M>0$, $s_0\leq s\leq c_1'h_\Lambda^{-1}$, and
\begin{equation}\label{lambda-restriction}
	h^\gamma_\Lambda\lambda^{-1/2}(q_\Lambda^{-d}h_\Lambda^d+1)
	\leq C_1,
\end{equation}
then
	with confidence $1-\tilde{\delta}$, there holds
\begin{eqnarray}\label{Error-est-noiseless}
	\|f_{D,\lambda}-f^*\|_\psi 
	&\leq &
	C_2  \big( \overbrace{M  \lambda^{-\frac{d+2\gamma\beta}{4\gamma}} |D|^{-1/2}\log\frac{3}{\tilde{\delta}}}^{\mbox{Stability}} 
	 + 
	\overbrace{
		\lambda^{\min\{\frac{1-\beta}2,v_g\}}s^{-\gamma}\mathcal I_{\{\alpha>3\}}+\lambda^{\min\{\frac{\alpha-\beta}{2},v_g\}}}^{\mbox{Fitting}} \big),
\end{eqnarray}
where $s_0,C_1,C_2$ are constants depending only on $d,\beta,\gamma$, $\|f^*\|_\varphi$ and $\alpha$ and $\mathcal I_{\mathcal A}$ denotes the indicator of the event $\mathcal A$.   
  
\end{theorem}  

It should be mentioned that $0\leq \beta\leq 1$ is natural, since $f_{D,\lambda}$ should belong to $\mathcal N_\psi$. 
The condition $\alpha\geq 1$ requires $f^*\in\mathcal N_\varphi\subseteq\mathcal N_\phi$. Different from Tikhonov regularization \cite{Feng2021radial} that specifies $1\leq \alpha\leq 3$, implying the  saturation phenomenon \cite{gerfo2008spectral} for Tikhonov regularization \cite{gerfo2008spectral}, our result holds for any $\alpha \geq 1$ and shows the excellent performance for spectral filters with large qualifications.

It can also  be  found in the above theorem that the stability error  depends on the magnitude of noise $M$, the size of data $|D|$, and the filter parameter $\lambda$. In particular, it decreases with respect to $\lambda$, implying that larger $\lambda$ leads to smaller stability error. However, the fitting error increases with $\lambda$, demonstrating that the stability error and fitting error are conflicting quantities in WSF. The best approximation is obtained when the stability error is comparable with the fitting error. This makes three restrictions on the selection of $\lambda$ in  \cref{Theorem:rate-spectral}. The first one is \eqref{lambda-restriction}, which presents a lower bound of $\lambda$, implying the necessity of introducing the high-pass filter function  in deriving stable estimator. The second one is  that $M  \lambda^{-\frac{d+2\gamma\beta}{4\gamma}} |D|^{-1/2}$ should be small, showing that the stability can be enhanced by adopting large $\lambda$. The third one, concerning the fitting performance, requires that  $\lambda^{\min\{\nu_g,\frac{\alpha-\beta}2\}}$ should   be small. 

 We then present several corollaries to illustrate the tightness of the derived approximation rates in  \cref{Theorem:rate-spectral}
and show that WSF enhances the stability of KI  without compromising its approximation performance.  
In the first corollary, we show that WSF possesses  the optimal approximation rates with achievable  $\lambda$, if the magnitude of noise $M$ is a constant.


\begin{corollary}\label{Corollary:approximation-spectral}
Let $\tilde{\delta}\in(0,1)$, $\hat{\phi}_k\sim k^{-2\gamma}$ with $2\gamma>d$ for any $k\in\mathbb N$ and     $\varphi$ be another   SBF  satisfying \eqref{kernel-relation} with   $1\leq \alpha\leq 3$.
If $f^*\in \mathcal N_\varphi$, $\{\varepsilon_i\}_{i=1}^{|D|}$ are zero-mean i.i.d.  random values satisfying $|\varepsilon_i|\leq M$ for some $M>0$,  $g_\lambda$ is the spectral filter satisfying \eqref{condition1} and \eqref{condition2} with qualification $\nu_g\geq 1$ and $\lambda\sim|D|^{-\frac{2\gamma}{2\alpha\gamma+d}}$, 
	and   
\begin{equation}\label{geo-distr-restriction}
	h^\gamma_{\Lambda}\tau_{\Lambda}^d\leq c_2'|D|^{-\frac{\gamma}{2\gamma\alpha+d}},
	\end{equation}  
 then with confidence $1-\tilde{\delta}$, there holds
\begin{eqnarray}\label{final-app-error-spectral}
	\|f_{D,\lambda}-f^*\|_{L^2(\mathbb S^d)} \leq
	C_3|D|^{-\frac{\gamma\alpha}{2\gamma\alpha+d}}\log\frac3{\tilde{\delta}},
	\end{eqnarray}
	where $C_3, c_2'$ are constants independent of  $|D|$  or $\tilde{\delta}$.
\end{corollary}

It can be found in \cite{lin2021subsampling} that the derived approximation error is optimal in the sense that there is a distribution $\rho^*$ of noise satisfying the condition of the above corollary,  a bad function $f^*_{bad}\in\mathcal N_{\varphi}$  with $\|f_{bad}^*\|_\varphi\leq U$  for some     $U>0$  and a constant $C_4$ independent of $|D|$ and $\tilde{\delta}$  such that
\begin{eqnarray}\label{lower-bound}
P_{\rho^*}\left[\|f_{D,\lambda}-f^*_{bad}\|_{L^2(\mathbb S^d)}\geq C_4|D|^{-\frac{\gamma\alpha}{2\gamma\alpha+d}}\right]
\geq
\frac{1}{4},
\end{eqnarray}
where $P_{\rho^*}$ denotes the probability with respect to the distribution $\rho^*$. The lower bound \eqref{lower-bound}, together with the upper bound \eqref{final-app-error-spectral}, shows that  the derived approximation rates cannot be essentially improved, and $|D|^{-\frac{\gamma\alpha}{2\gamma\alpha+d}}$ is the optimal approximation rate for noisy data under the aforementioned assumptions. Furthermore, \eqref{geo-distr-restriction} shows that if the noise is large, it is not necessary to force the samples set to be $\theta$-quasi-uniform, in which 
$h^\gamma_\Lambda\tau_\Lambda^d\leq c_2'|D|^{-\frac{\gamma}d}< c_3'|D|^{-\frac{\gamma}{2\gamma\alpha+d}}$ for some $c_3'>0$.
If we impose $\Lambda$ to be $\theta$-quasi-uniform, the following corollary shows that with appropriately selected spectral filters,  optimal approximation rates hold for any $\alpha\geq 1$.
\begin{corollary}\label{Corollary:approximation-spectral-large}
Under the condition of  \cref{Corollary:approximation-spectral}, if the qualification of $g_\lambda$ satisfies $\nu_g\geq \alpha$ and $\Lambda$ is $\theta$-quasi-uniform for some $\theta>1$, then 
\begin{eqnarray}\label{final-app-error-spectral-large}
	\|f_{D,\lambda}-f^*\|_{L^2(\mathbb S^d)} \leq
	C_3|D|^{-\frac{\gamma\alpha}{2\gamma\alpha+d}}\log\frac3{\tilde{\delta}},\qquad \forall \alpha\geq 1.
 \end{eqnarray}
\end{corollary}

Recalling that the qualification of Landweber iteration and spectral cut-off are infinite, \cref{Corollary:approximation-spectral-large} presents the advantage for these two types of filters in circumventing the saturation of Tikhonov regularization.
In the next corollary, we present that the derived bound \eqref{Error-est-noiseless} is rate optimal when the noise is small and $\Lambda$ is $\theta$-quasi-uniformly distributed, provided  $\lambda$ is appropriately tuned.
\begin{corollary}\label{Corollary:approximation-noise-free-1}
Let $\tilde{\delta}\in(0,1)$. Under the condition of  \cref{Corollary:approximation-spectral},  
	if  $M\leq   |D|^{-\frac{\gamma}{d}}$, $\Lambda$ is $\theta$-quasi-uniform, and $\lambda\sim |D|^{-\frac{2\gamma}{d}}$, then with confidence $1-\tilde{\delta}$,  there holds
\begin{eqnarray}\label{app-noise-free-in-coro}
	\|f_{D,\lambda}-f^*\|_{L^2(\mathbb S^d)} \leq
	C_5|D|^{-\frac{\gamma\alpha}{d}}\log\frac3{\tilde{\delta}}
\end{eqnarray}
	where $C_5$ is a constant independent of  $|D|$, $M$ or $\tilde{\delta}$.
\end{corollary}

 \cref{Corollary:approximation-noise-free-1} presents the approximation capability of WSF in handling noise-free data, i.e., $M=0$.
It is well known \cite{narcowich2002scattered,narcowich2007direct} that for noise-free data, the derived approximation rate $|D|^{-\frac{\gamma\alpha}{d}}$ cannot be essentially improved, showing that WSF does not degrade the approximation performance of KI. Comparing \eqref{final-app-error-spectral} with \eqref{app-noise-free-in-coro}, we highlight that the worse rate in \eqref{final-app-error-spectral} is due to the noise rather than the fitting scheme.

The above results verify the excellent performance of WSF with suitable filter parameters. 
 In our next theorem, we present an approximation error estimate for Ada-WSF, $f_{D^{tr},\hat{\lambda}}$.

\begin{theorem}\label{Theorem:parameter}
	Let $\tilde{\delta}\in(0,1)$, $\hat{\phi}_k\sim k^{-2\gamma}$ with $2\gamma>d$ for any $k\in\mathbb N$ and     $\varphi$ be another   SBF  satisfying \eqref{kernel-relation} with   $\alpha\geq 1$.
	If $f^*\in \mathcal N_\varphi$, $\{\varepsilon_i\}_{i=1}^{|D|}$ are zero-mean i.i.d. random values satisfying $|\varepsilon_i|\leq M$ for some $M>0$,  $\Lambda$ is $\theta$-quasi-uniform,
	  $|D^{tr}| = |D^{val}| = |D|/2$, and $\hat{\lambda}$ is defined by \eqref{CV-for-parameter}, then 
	with confidence at least $1-\tilde{\delta}$, there holds
\begin{eqnarray*} 
      &&\|f_{D^{tr},\hat{\lambda}}-f^*\|^2_{L^2(\mathbb S^d)}
     \leq  
  \min_{\lambda\in\Xi_L}\|f_{D^{tr},\lambda}-f^*\|_{L^2(\mathbb S^d)}
   + 
   C_6|D|^{-\frac{\gamma}{d}}\max_{\lambda\in\Xi_L}\left(\lambda^{-\frac{d+2\gamma}{2\gamma}} |D|^{-1}+ \lambda^{\alpha-1}\right)\\
   &+&
   C_6 \frac1{|D|}\log^3\frac{3L}{\tilde{\delta}}\max_{\lambda\in\Xi_L}\left(\lambda^{-\frac{d+2\gamma}{4\gamma}} |D|^{-1/2} 
	+
		\lambda^{\frac{\alpha-1}2}+\lambda^{-\frac{d+2\gamma}{2\gamma}} |D|^{-1}+ \lambda^{\alpha-1}\right),
\end{eqnarray*}
where $C_6$ is a constant depending only on $M,d,\phi,\gamma$ and $\|f^*\|_\varphi$. 
\end{theorem}

The above theorem demonstrates the effectiveness of the proposed weighted cross-validation approach.
We then derive the  corollary based on \cref{Theorem:parameter} to show the feasibility and optimality  of the proposed weighted cross-validation strategy. 



\begin{corollary}\label{corollary:parameter}
	Let $\tilde{\delta}\in(0,1)$. Under the assumptions in  \cref{Theorem:parameter}, if  $|D^{tr}| = |D^{val}| = |D|/2$ and $\lambda\in\Xi_L:=\{1,1/2,\dots,1/\sqrt{|D|}\}$ and $\gamma>3d/2$, then
	with confidence at least $1-\tilde{\delta}$,  there holds
	$$ 
	\|f_{D^{tr},\hat{\lambda}}-f^*\|_{L^2(\mathbb S^d)}^2\leq
	C_7|D|^{-\frac{2\gamma\alpha}{2\gamma\alpha+d}} \log\frac{3}{\tilde{\delta}},
	$$ 
where $C_7$ is a constant independent of $|D|$ or $\tilde{\delta}$.
\end{corollary}

\subsection{Related work}
KI is a classical and long-standing approach in tackling spherical data. In particular, for noise-free data, \cite{narcowich1998stability,levesley1999norm} derived  the condition number of the kernel matrix and showed their concerns on the stability issue of KI; \cite{mhaskar1999approximation,narcowich2002scattered,levesley2005approximation} deduced the approximation error of KI in terms of the mesh norm $h_\Lambda$;
\cite{narcowich2007direct,hangelbroek2010kernel,hangelbroek2011kernel}    made significant advancement in expanding the approximation power of  KI beyond the native space setting. All these interesting results showed that KI is feasible and efficient  to handle spherical data, provided the geometric distribution of  $\Lambda$ is not too strange and their corresponding outputs are sampled without noise. 

Noisy data fitting on spheres has recently attracted increasing attentions \cite{hesse2017radial,hesse2021local,Feng2021radial,lin2021subsampling,lin2023dis}, in which  \cite{hesse2017radial,Feng2021radial,lin2023dis} is the most  related with our study. Using a novel sampling inequality on the sphere, \cite{hesse2017radial} derived tight approximation error estimates for Tikhonov regularization. There are mainly three differences between our work and \cite{hesse2017radial}. At first, as shown in \eqref{Model1:fixed}, our studied model contains both large and small random noise while \cite{hesse2017radial} focused on extremely small noise like in  \cref{Corollary:approximation-noise-free-1}. As a result, we adopted the integral operator approach \cite{Feng2021radial} rather than the sampling inequality for error analysis.   Then, we are concerned with spectral filter  approaches that contain a large number of algorithms, while \cite{hesse2017radial} only focused on
Tikhonov regularization. At last, our result in  \cref{Theorem:rate-spectral} accommodates Sobolev error estimator and adapts to $f^*\in\mathcal N_\varphi{\subseteq}\mathcal N_\phi$, while \cite{hesse2017radial} only devotes to error analysis in $L^2(\mathbb S^d)$ and $f^*\in\mathcal N_\phi$. 

Comparing with \cite{Feng2021radial}, there are also two novelties in our results, though their models and proof skills are similar. On one hand, we focused on spectral filter algorithm while \cite{Feng2021radial} only considered Tikhonov regularization. It was shown in \cref{Corollary:approximation-spectral-large} that both Landweber iteration and spectral cut-off succeed in circumventing the saturation of Tikhonov regularization. On the other hand, \cite{Feng2021radial} derived optimal approximation error for Tikhonov regularization when $M$
in model \eqref{Model1:fixed} is a constant, but the approximation rate is suboptimal for small $M$. In particular, under the same conditions as  \cref{Corollary:approximation-noise-free-1}, the approximation rate derived in \cite{Feng2021radial} is $\mathcal O(|D|^{-\frac{\gamma\alpha}{2d}})$, which is far worse than the optimal one $\mathcal O(|D|^{-\frac{\gamma\alpha}{d}})$ and implies that adding regularization term may degrade the approximation of KI. Finally, we compare our results with those in \cite{lin2023dis}, where a divide-and-conquer technique is proposed to enhance the stability of KI. There are roughly three differences. Firstly, the algorithms are different. In particular, \cite{lin2023dis} enhanced the stability via a divide-and-conquer approach but we do the same thing by using spectral filters. Secondly, the approximation rates derived in \cite{lin2023dis} for the divide-and-conquer approach is suboptimal while our derived rates are optimal. Thirdly, we provide adaptive strategy to select the filter parameter while there lacks similar results for selecting the splitting parameter in \cite{lin2023dis}.  

Spectral filter approach is not new in spherical data analysis. In fact, the well known filtered  hyperinterpolation \cite{sloan2012filtered,an2012regularized,an2021lasso,lin2021subsampling} showed that adding low-pass filters to some spherical polynomials is capable of improving its   localization  in the spatial domain. It should be highlighted that the spectral filters for KI are actually high-pass that are essentially different from those in the filtered hyperinterpolation. Another line of similar work is the spectral algorithms in machine learning \cite{gerfo2008spectral,dicker2017kernel,guo2017learning,mucke2018parallelizing}, where similar spectral filters are used to tackle with i.i.d. samples. In particular, \cite{dicker2017kernel,guo2017learning} derived optimal learning rates for such spectral algorithms in the framework of statistical learning theory. We highlight  again that it is highly non-trivial to borrow the idea of spectral filters from  i.i.d. samples analysis to spherical data fitting, both in algorithmic designation and theoretical analysis. From the algorithm side, deterministic sampling requires to force the algorithm to distinguish different points on the sphere that is not considered for i.i.d. samples. From the theoretical analysis side,  deterministic sampling makes the existing statistical tools such as large number law and some important concentration inequalities unavailable, and requires exclusive tools like the spherical quadrature rule and spherical harmonic analysis.

\section{Numerical Experiments}\label{Sec.Numerical}
In this section, we conduct both toy simulations and two real-world data trials to show the power of WSF. 
All the experiments are run through Python 3.7 on a PC equipped with an Intel Core i5 2 GHz processor. 
The Python code, as well as the adopted data sets, can be downloaded from \url{https://github.com/Ariesoomoon/Weighted-Spectral-Filters-on-Spheres.git}.


\subsection{Toy simulations}

We conduct six toy  simulations to show the feasibility and verify the improved prediction performance of WSF. 
The first one aims to demonstrate the role of the high-pass filters in fitting noisy data. The second one illustrates the advantage of deterministic sampling over the random sampling. The 
third one is designed to show the effectiveness of the proposed weighted cross-validation scheme for parameter selection.
The fourth one indicates that WSF succeeds in solving  the instability issue of KI.  
The fifth one shows the advantages of WSF under the uniform measure. 
The last one visualizes the performance of WSF in fitting noisy data with different noise levels.

The samples are generated via \eqref{Model1:fixed} with
$\varepsilon_i$ following the truncated  Gaussian distribution $\mathcal{N}\left(0, \delta^2\right)$ for $\delta\in\{0.1, 0.3, 0.5\}$ and 

\begin{equation}\label{Model_toysimulation}
    f^*(x)=\sum_{i=1}^6 \tilde{\psi}\left(\left\|x-z_i\right\|_2\right),
\end{equation}
where $\tilde{\psi}(u)$ is the well-known Wendland function
\begin{equation}\label{Wendland_function}
    \tilde{\psi}(u)=(1-u)_{+}^8\left(32 u^3+25 u^2+8 u+1\right),
\end{equation}
$u_{+} = \max\{u, 0\}$, $z_i \in \mathbb{S}^2$  $(i=1, \dots,6)$ are the center points   $(1,0,0)$, $(-1, 0,0), (0,1,0), (0, -1, 0),(0,0,1), (0,0,-1)$,   and the truncated value of Gaussian distribution is 2.5, i.e., the noise is set to be $2.5\mbox{sgn}(\varepsilon_i)$ if $|\varepsilon_i|\geq 2.5$. 
To fully demonstrate the effectiveness of WSF and explore the effect of different sampling mechanisms, training samples $\{x_i\}_{i=1}^{|D|}\subset\mathbb S^2$ are generated in three ways:
(1) random sampling to generate the equidistributed points on $\mathbb S^2$ \cite{deserno2004generate};
(2) Womersley's symmetric spherical $t$-designs \cite{chen2018spherical} with $|D|= t^2/2 + t/2 + O(1)$ and $t$ varies in the range of $\{3, 7, 11, \dots, 63\}$; (3)    generating 120 points $\{x_i\}_{i=1}^{120}$ from spherical 15-designs and then rotating them according to rotation matrix
$$
A_k:=\left(\begin{array}{ccc}
\cos (k \pi / 20) & -\sin (k \pi / 20) & 0 \\
\sin (k \pi / 20) & \cos (k \pi / 20) & 0 \\
0 & 0 & 1
\end{array}\right)
$$
for $k=1, 2,\cdots, 19$, i.e., $x_{i,k}=A_k {x}_i$  in $k$th rotation.
That is, there are totally $120(k+1)$ input points after $k$ rotations. These three input generating ways are called the random way,  $t$-design way, and rotated way, respectively. 
The validation data set are generated in $t$-design way with $t=47$.
There are $N'=4000$ testing samples for Simulation 1-5 and $N'=6000$ for Simulation 6  generated according to $y'_j=f^*\left(x_j'\right)$, 
where $x_j'$ are drawn i.i.d. according to  the uniform distribution in the (hyper)cube $[-1,1]^{3}$ with each $x_j'$ multiplied by the scale factor $\frac{1}{||x_j'||_2}$.
The adopted  kernel function (positive definite)   is $\phi(x\cdot x')=h(\|x-x'\|_{2})$ with 
$$
   h(\|x\|_{2}):=
   \left\{\begin{array}{ll}
   (1-\| x \| _{2})^{4}(4\| x\| _{2} +1), &  0 < \| x \| _{2}\leq 1,x \in \mathbb{R}^{3}, \\
 0, & \| x \| _{2} >1.
 \end{array}
 \right.
$$
It can be found in \cite{wendland2004scattered} and \cite{narcowich2007approximation} that $\hat{\phi}_k\sim k^{-4}$ and $f^*\in \mathcal N_\varphi$ with $\varphi$ and $\phi$ satisfying \eqref{kernel-relation} with $\alpha=2$.  

We adopt the RMSE, defined by $RMSE(\hat{f}_D, y')$:=$\sqrt{\frac{1}{N'}\sum_{j=1}^{N'}(\hat{f}_D(x'_j)-y'_j)^2}$ 
on testing samples, to evaluate prediction performance, where $\hat{f}_D(x'_j)$ represents the predicted output of input point $x'_j$ based on the training data set $D$. Another error measure, infinity norm ($\|\cdot\|_{\infty}$), is also adopted in our Simulation 5, with
$\|\hat{f}_D-y'\|_{\infty}:=\max_{i=1,\dots,N'}|\hat{f}_D(x'_i)-y'_i|.$
We conduct each simulation five times for averaging. 

  \vspace{0.15in}
\textbf{Simulation 1: Role of  spectral filters.} In this simulation, we show the role of spectral filters in affecting the stability, fitting and approximation performance. Taking $\{x_i\}_{i=1}^{|D|}$ generated in  rotated way with the  $k=10$, we show that WSF significantly reduces the CNKM (condition number of kernel matrix) in  \cref{fig:CNKM_rotation}.

\begin{figure}[H]
  \vspace{-0.2in}
	\centering
	\subfigure[Tikhonov regularization]{\includegraphics[scale=0.31]{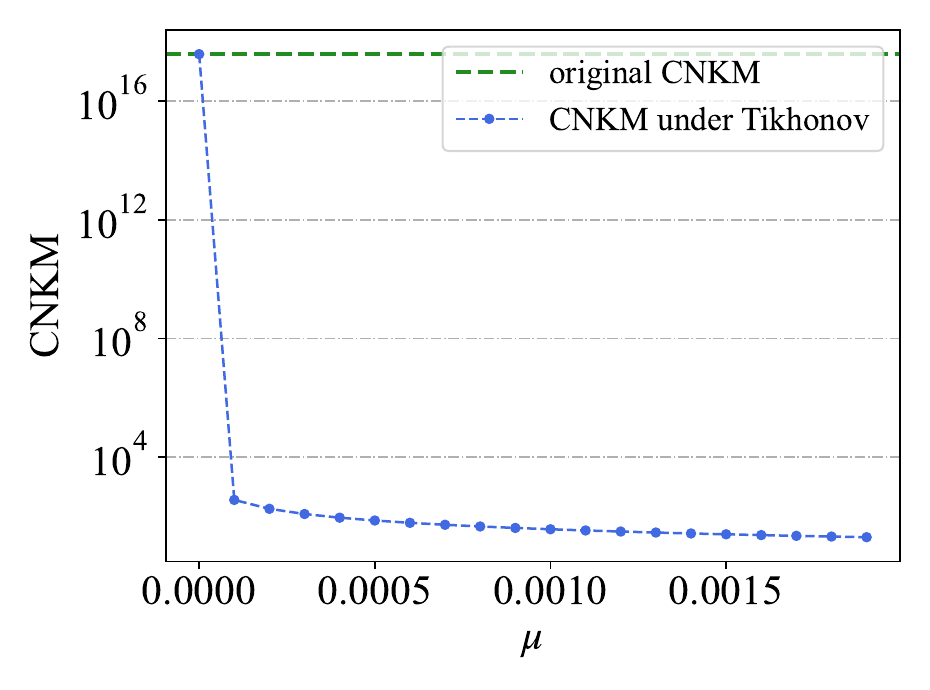}} 
	\subfigure[Landweber iteration]{\includegraphics[scale=0.31]{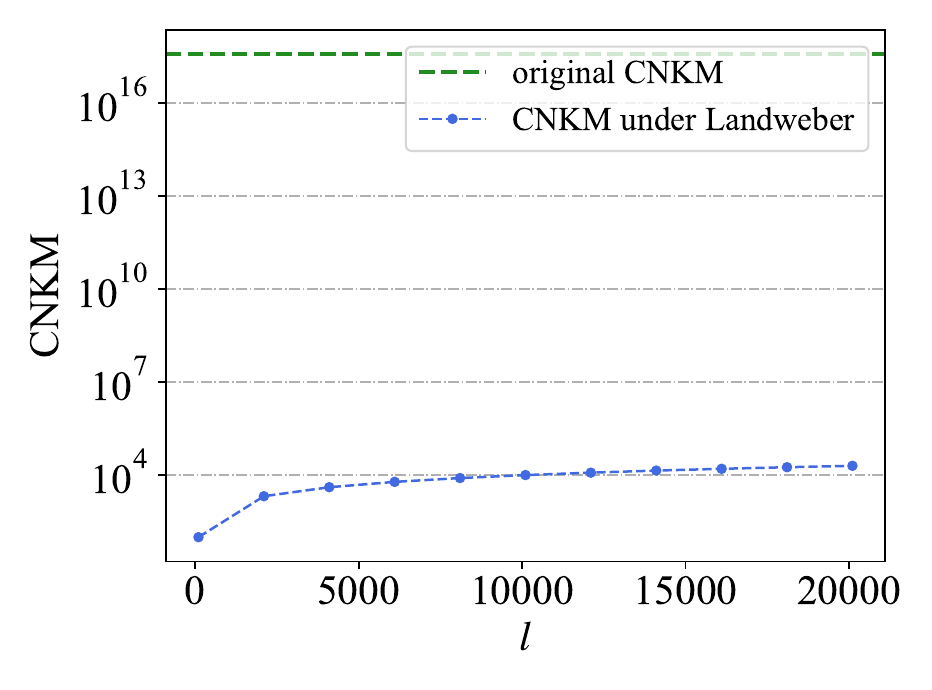}}
	\subfigure[Spectral cut-off]{\includegraphics[scale=0.31]{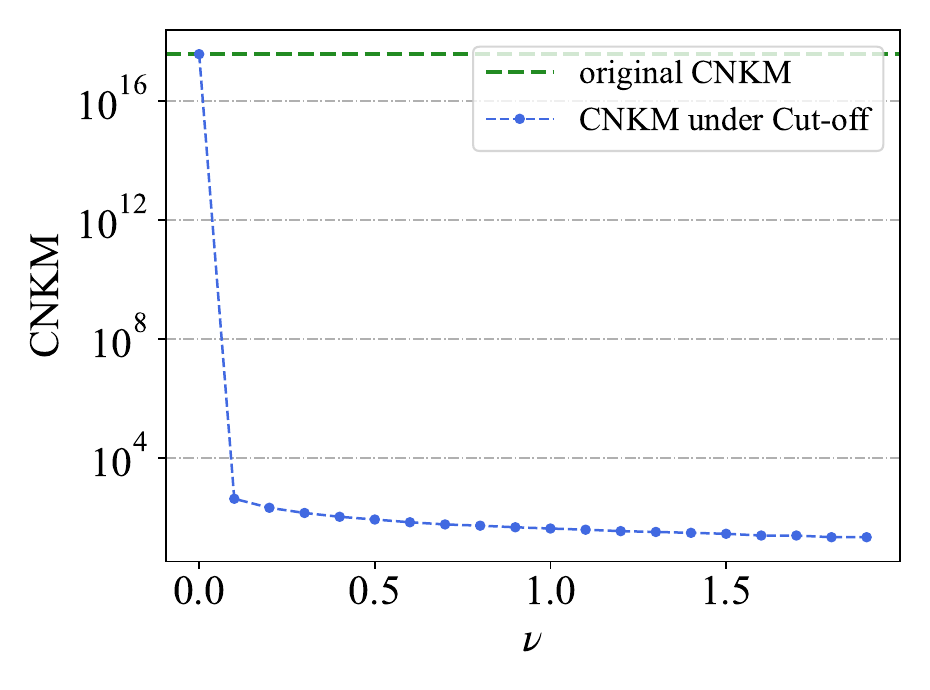}} 
   \vspace{-0.05in}
	\caption{WSF to reduce the CNKM for rotated data with $|D|=1200$ 
 and truncated Gaussian noise  with standard deviation $\delta=0.5$.}
	\label{fig:CNKM_rotation}
	  \vspace{-0.15in}
\end{figure}

 \cref{fig:CNKM_rotation} shows that all WSFs succeed  in stabilizing KI in terms of reducing the CNKM. Furthermore, CNKM decreases as the filter parameters, i.e., $\mu$ in Tikhonov regularization, $1/l$ in Landweber iteration and $\nu$ in spectral cut-off,  increase. However, reducing CNKM is not enough to show the excellent performance of WSF. In the following, we show the  role of filter parameter in stability, fitting and approximation error, respectively, under different sampling mechanisms. Specifically, $t$ is set to 47 in $t$-design way to generate 1130 training samples, the number of rotation $k$ is set to 10 to generate 1200 training samples, and the random way generates 1130 training samples. The numerical results are shown in \cref{fig: bias} and \cref{fig: stability_approximation}.

  \vspace{-0.15in}
\begin{figure}[H]
	\centering
	\subfigure[Random ($|D|=1130$)]{\includegraphics[scale=0.31]{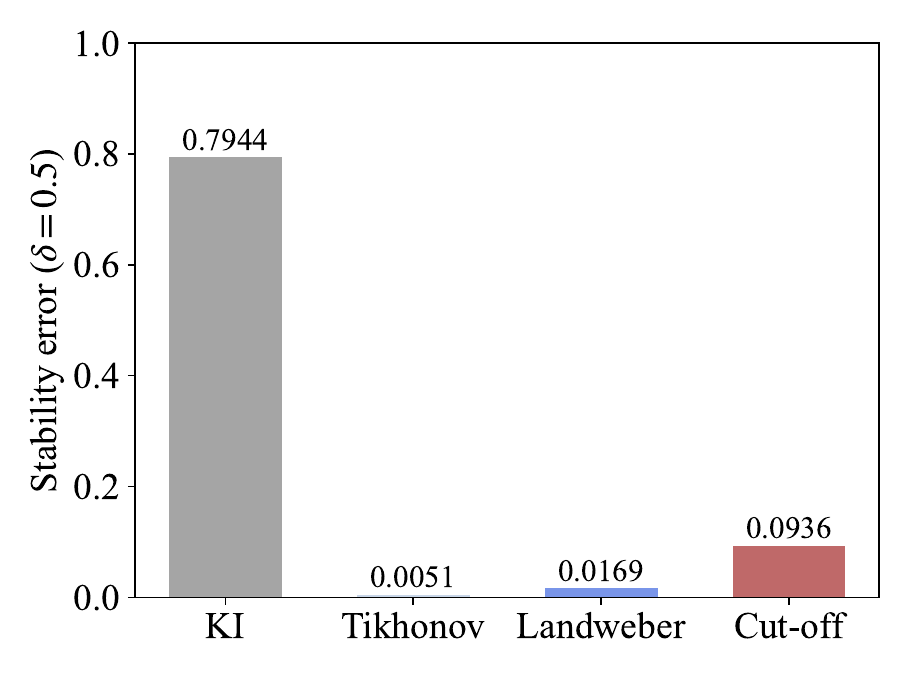}} 
	\subfigure[$t$-design ($|D|=1130$)]{\includegraphics[scale=0.31]{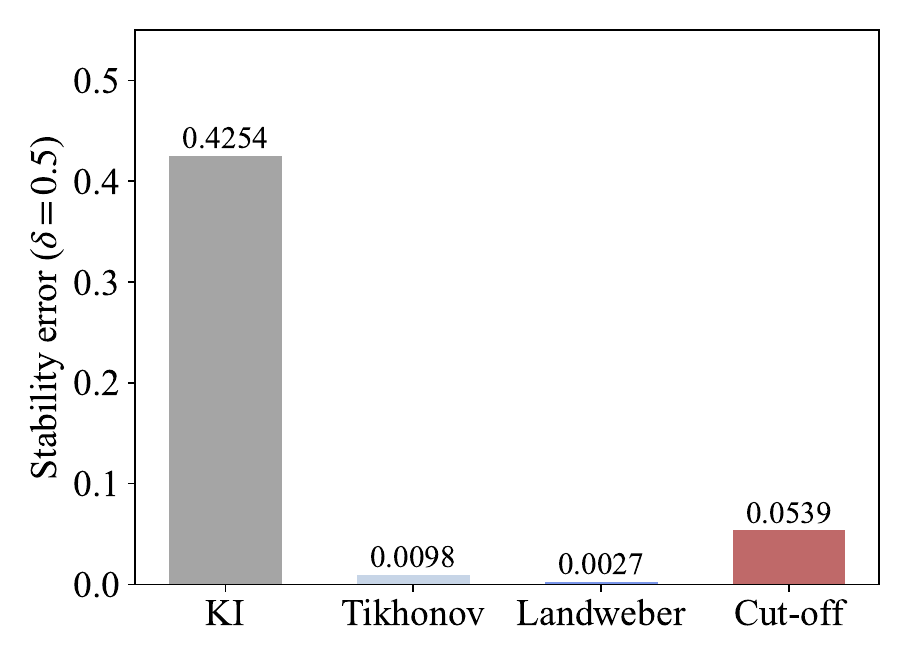}}
	\subfigure[Rotation ($|D|=1200$)]{\includegraphics[scale=0.31]{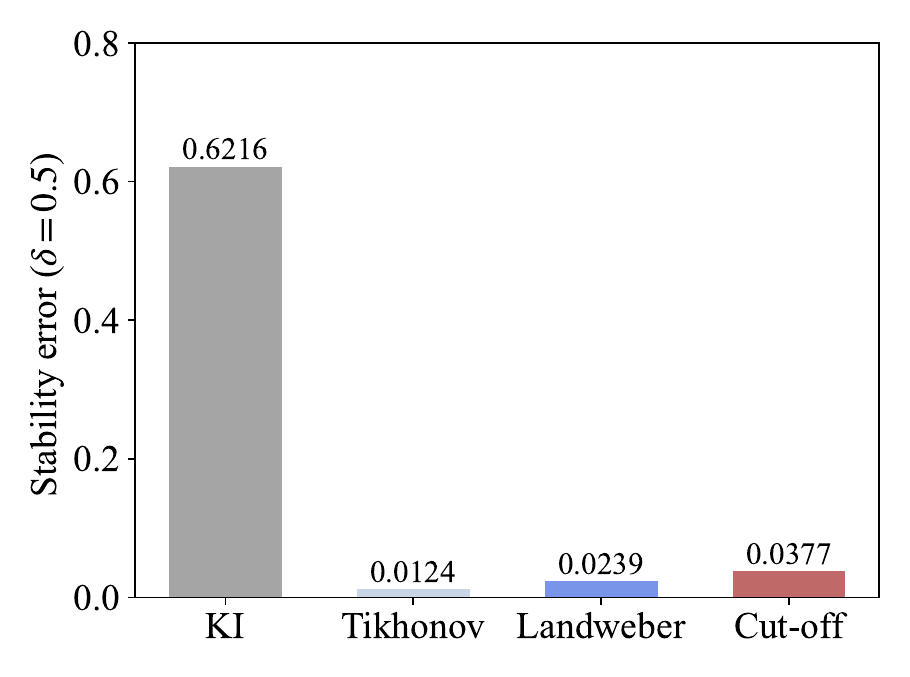}} 
   \vspace{-0.15in}
	\caption{WSF to reduce the stability error with truncated  Gaussian noise with  $\delta=0.5$.}
	\label{fig: bias}
\end{figure}
  \vspace{-0.2in}

  As shown in \cref{fig: bias}, WSF significantly reduces the stability error (stability part of approximation error), provided filter parameters 
  are appropriately tuned, which verifies our estimates of the stability error in  \cref{Theorem:rate-spectral}. However, \cref{fig: stability_approximation} shows that the fitting error increases with respect to the filter parameter, showing that reducing the CNKM in  \cref{fig:CNKM_rotation} cannot guarantee the good fitting capability. This verifies our fitting error estimate in  \cref{Theorem:rate-spectral}. 
 As shown in \cref{fig: stability_approximation}, the approximation rate decreases with the filter parameter at first and then increases, demonstrating that there is an optimal filter parameter larger than 0 to optimize the approximation performance of WSF. We also find that the approximation capability of WSF is stable with respect to the filter parameter in the sense that small change of the filter parameter does not lead to dramatic change of the approximation error. Furthermore, the best filter parameter can be achieved  near the one that the stability error equals to the fitting error.  

 \begin{figure}[H]
	\centering
	\includegraphics[scale=0.45]{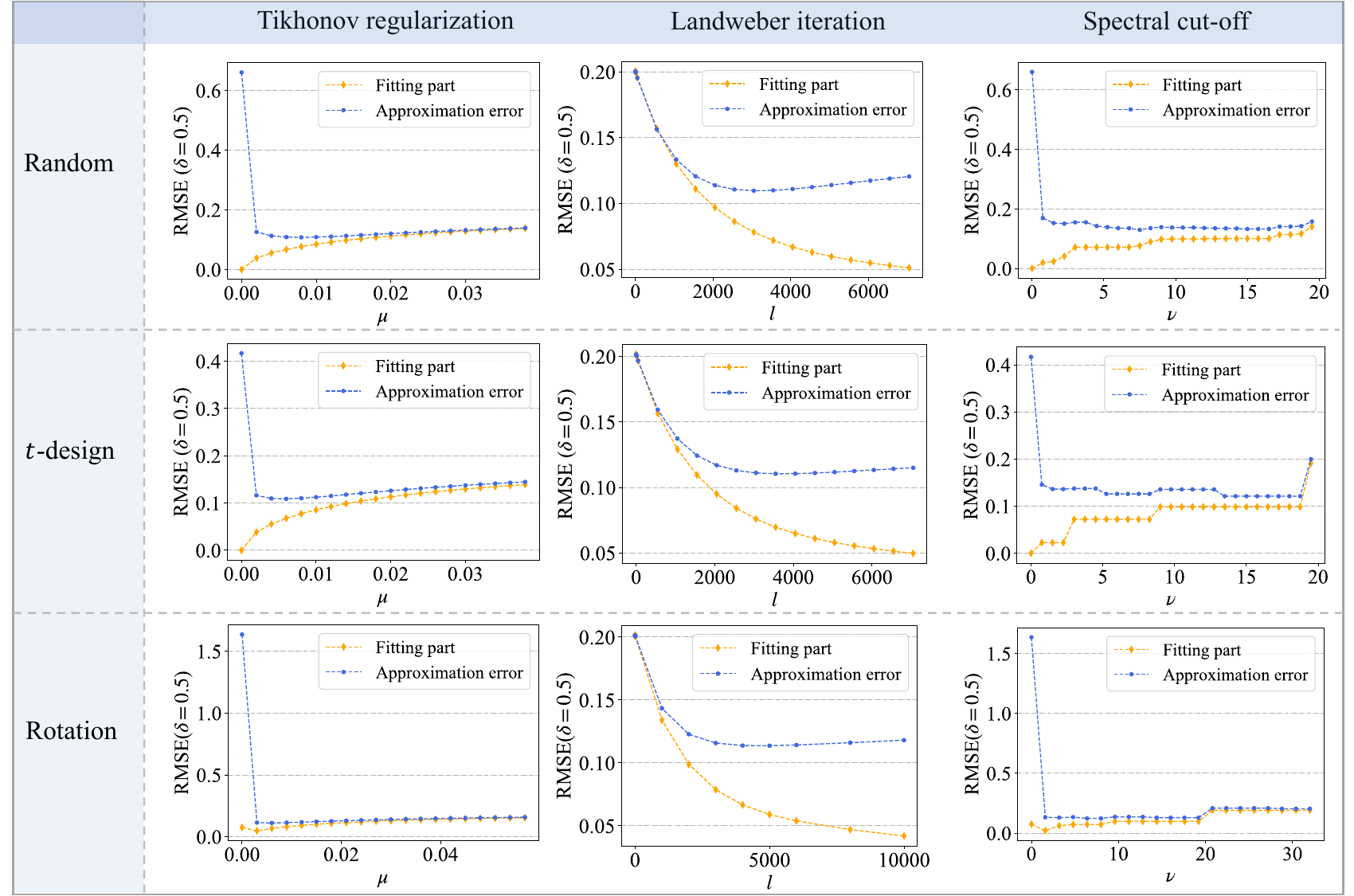}
   \vspace{-0.05in}
	\caption{Fitting error and approximation error of WSF change along with filter parameters with truncated Gaussian noise with standard deviation $\delta=0.5$.}
	\label{fig: stability_approximation} 
	\vspace{-0.1in}
\end{figure}
\vspace{-0.3cm}

  \vspace{-0.15in}
\begin{figure}[H]
	\centering
	\subfigure{\includegraphics[scale=0.3]{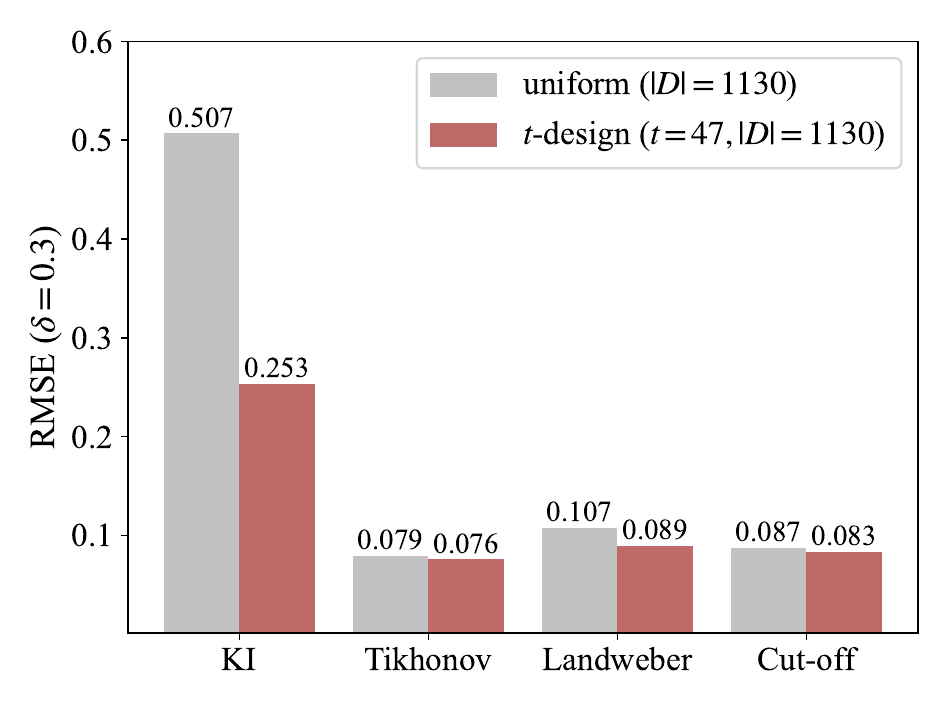}
\label{fig:sample_role_noise03} } 
 	\subfigure{\includegraphics[scale=0.3]{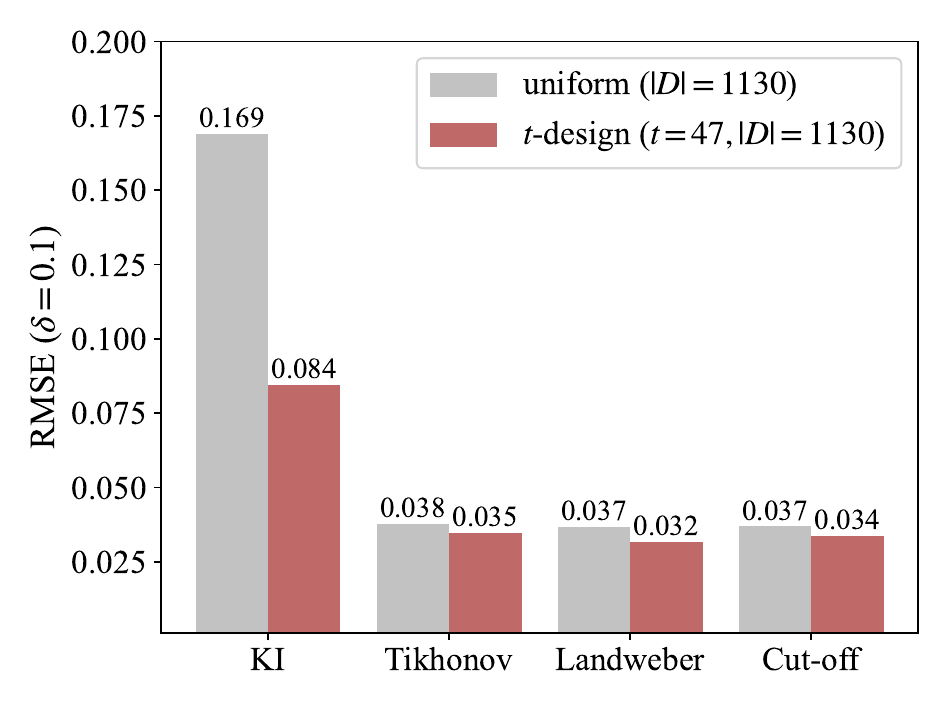}
  \label{fig:sample_role_noise01}}
  	\subfigure{\includegraphics[scale=0.3]{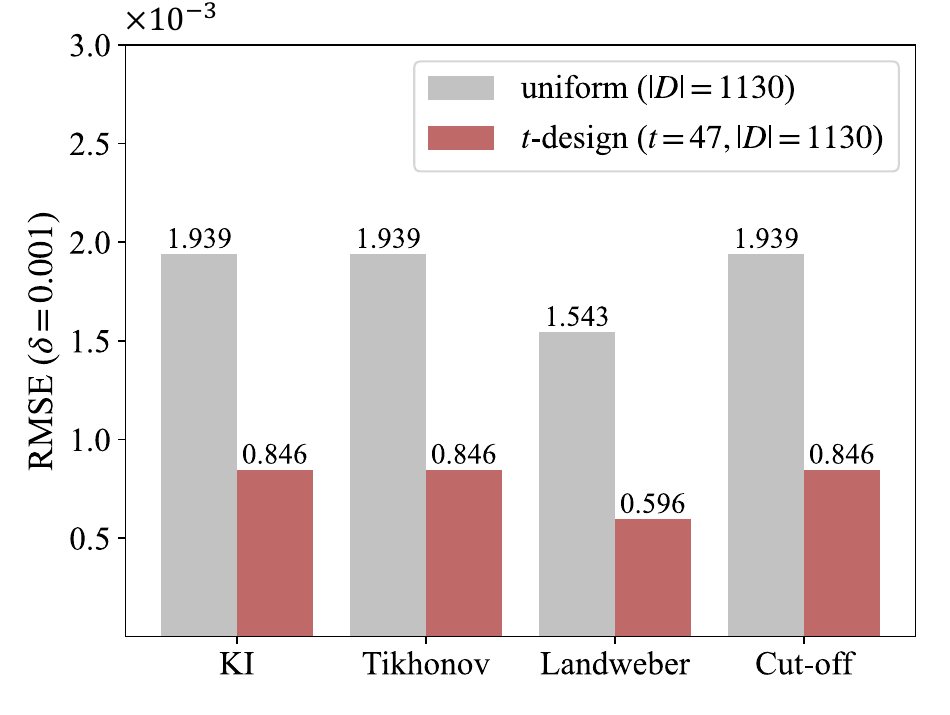}
   \label{fig:sample_role_noise001}} 
   \vspace{-0.15in}
	\caption{Approximation performance of WSF when training samples are drawn by different sampling mechanisms (random way and $t$-design way with $t=47$, $|D|=1130$). Note that the ordinate of the rightmost subplot is the RMSE multiplied by $10^{-3}$.}
	\label{fig:sample_role}
	\vspace{-0.1in}
\end{figure}

\textbf{Simulation 2: Role of  sampling mechanisms.} In this simulation, we study the impact of the sampling mechanism on approximation performance of KI and WSF under different noise levels.  \cref{fig:sample_role} shows the  results of different approaches in $t$-design way with $t=47$ 
and in random way with same data size as in $t$-design way, respectively. The validation data set is generated in $t$-design way with $t=45$.

There are four interesting observations  in \cref{fig:sample_role}. Firstly, the sampling mechanism plays a crucial role in KI. The main reason is that CNKM depends heavily on $q_\Lambda$ \cite{narcowich1998stability} that would be changed dramatically for different sampling mechanisms. Secondly, the sampling mechanism does not affect the approximation performance of WSF much. As shown in  \cref{Corollary:approximation-spectral}, the optimal approximation performance can be realized by numerous geometric distributions of $\Lambda$ and it is easy to verify that both the  $t$-design and random sampling way used in this study  satisfy the condition \eqref{geo-distr-restriction} with high probability. 
Thirdly, WSF outperforms KI under both sampling mechanisms in relatively high noise cases (\cref{fig:sample_role_noise03} and \cref{fig:sample_role_noise01}), which verifies the  \cref{Corollary:approximation-spectral} and shows the necessity of introducing the spectral filter for KI. Finally, WSF does not degrade the approximation performance of KI in the lower noise case (\cref{fig:sample_role_noise001}), which verifies the \cref{Corollary:approximation-noise-free-1}.



	\vspace{0.1in}
\textbf{Simulation 3: Effectiveness of weighted parameter selection.} 
This simulation studies the effectiveness of the proposed weighted parameter selection strategy. We take $t$-design and random way to generate 1130 training samples, respectively, with the noise level $\delta=0.5$. Let $D^{tr}=\{(x_i^{tr}, y_i^{tr})\}$,  $D^{val}=\{(x_i^{val}, y_i^{val})\}$, and $D^{test}=\{(x_i^{test}, y_i^{test})\}$ be the training, validation, and testing data set, respectively.  
Denote $RMSE_{\hat{\lambda}}$ the approximation error of WSF, which runs first on  $D^{tr}$ and then determines filter parameter $\hat{\lambda}\in\Xi_L$ based on $D^{val}$ via weighted parameter selection strategy as stated in \eqref{CV-for-parameter}. 
Denote $RMSE_{\lambda'}$ the approximation error of WSF with filter parameter $\lambda'\in\Xi_L$ selected directly based on $D^{test}$.
 It should be mentioned that $\lambda'$ provides a baseline to measure the optimality of $\hat{\lambda}$.

Table  \ref{tab: weight parameter selection} shows that for each WSF, $RMSE_{\hat{\lambda}}$ nearly equals $RMSE_{\lambda'}$, verifying the effectiveness of the proposed parameter selection approach. This versifies \cref{Theorem:parameter} and \cref{corollary:parameter} and demonstrates the optimality of the proposed weighted cross-validation scheme.

\vspace{-0.05in}
\begin{table}[H]
\small
  \renewcommand{\arraystretch}{1.2}
  \centering
 \begin{tabular}{c|ccc|ccc}
  \hline
  \multirow{2}{*}{} & \multicolumn{3}{c|}{\textbf{$t$-design}}&\multicolumn{3}{c}{\textbf{Random}} \\  
 
  & $RMSE_{\hat{\lambda}}$     & $RMSE_{\lambda'}$    & $\Delta RMSE$ & $RMSE_{\hat{\lambda}}$     & $RMSE_{\lambda'}$    & $\Delta RMSE$  \\ 
  \hline
  Tikhonov   & 0.1056  & 0.1096  & 4.0$\times 10^{-3}$  & 0.1067   & 0.1100  & 3.3$\times 10^{-3}$    \\ 
  Landweber & 0.1069 & 0.1078   & 8.5$\times 10^{-4}$ & 0.1081   & 0.1082  &  8.1$\times 10^{-5}$  \\
  Cut-off  & 0.1230  & 0.1269   & 3.9$\times 10^{-3}$ & 0.1403   & 0.1427  &  2.4$\times 10^{-3}$   \\   \hline
 \end{tabular}
 \vspace{-0.05in}
   \caption{Effectiveness of weighted parameter selection.
   }
   \label{tab: weight parameter selection}
\end{table}
\vspace{-0.1in}

\textbf{Simulation 4: Outperformance of WSF.} In this simulation, we show  that the above three WSFs are able to circumvent the instability issue of KI via showing the relation between the RMSE and data size with different noise levels. As shown in  \cref{fig: drawbackofKI_tdesign} and  \cref{fig: drawbackofKI_rotation} in Section \ref{Sec.Stability}, KI has   critical limitation in spherical data analysis in terms that its approximation rate does not decrease with the data size. We show in   \cref{fig: MSE_N} that adding spectral filters succeeds in breaking through this limitation. 
In particular, 
 \cref{fig: MSE_N} shows that RMSE of Tikhonov regularization, Landweber iteration and spectral cut-off decreases with the growing data size in random, $t$-design, and rotated way. 

 There are three interesting observations in \cref{fig: MSE_N}: (1) All the three WSFs enable the RMSE of KI to decrease with respect to $|D|$, which verifies our theoretical assertions in  \cref{Theorem:rate-spectral} and  \cref{corollary:parameter}; (2) WSF is able to circumvent the instability issue of KI in the sense that it is efficient for even large noise $\delta=0.5$; (3) WSF is effective for different sampling mechanisms, which is totally different from KI. In fact, as shown in  \cref{fig:sample_role}, the prediction performance of KI changes dramatically when the sampling mechanism changes. The stability with respect to the sampling mechanism also  demonstrates the power of spectral filters and verifies the condition \eqref{geo-distr-restriction} in  \cref{Corollary:approximation-spectral}. 

 


\vspace{-0.2in}
\begin{figure}[H]
	\centering
	\includegraphics[scale=0.5]{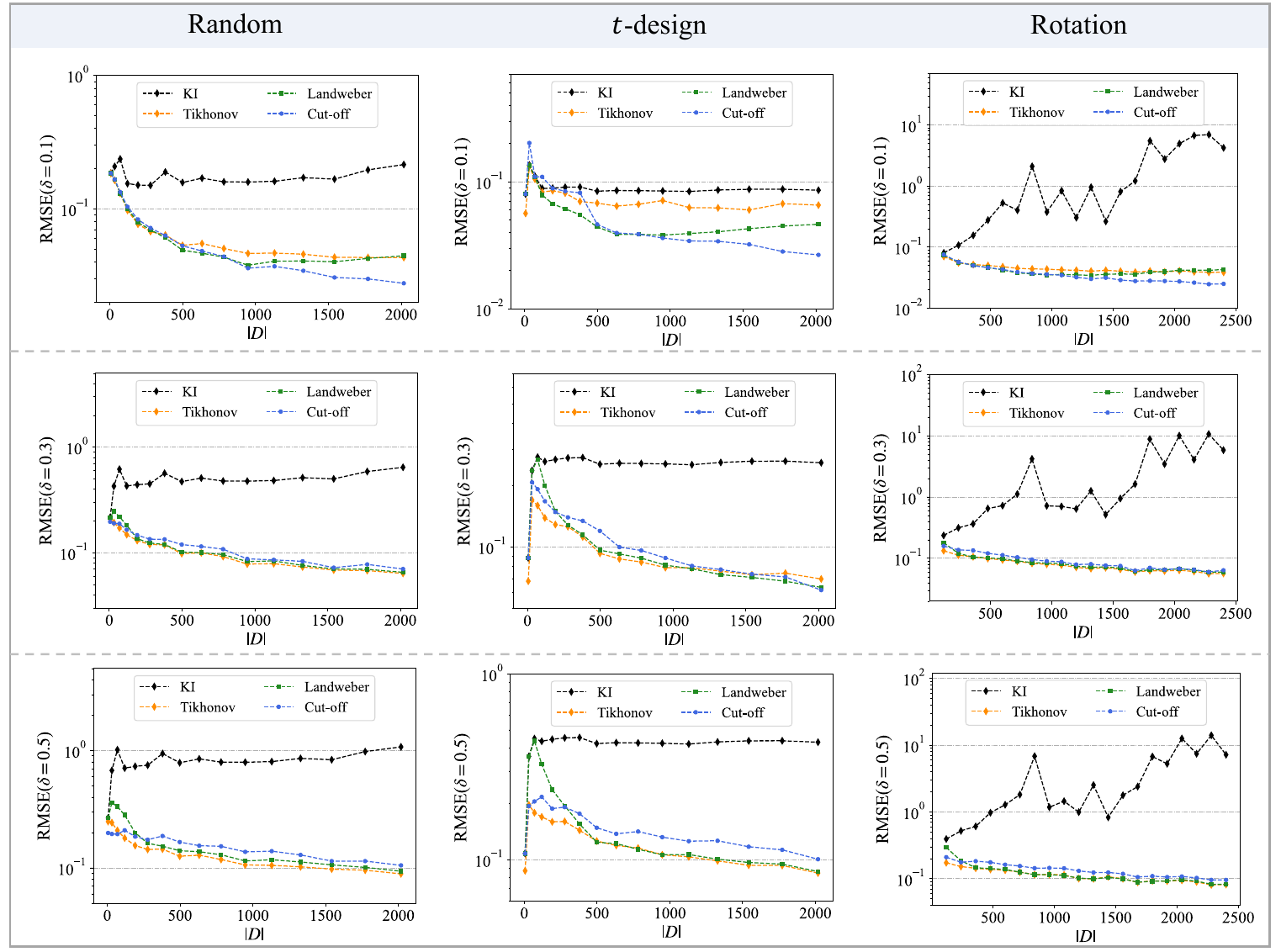}
  \vspace{-0.05in}
	\caption{Relations between the RMSE and training sample size of WSF and KI when training samples are drawn by different sampling mechanisms with noise level $\delta=0.1, 0.3,0.5$.}
	\label{fig: MSE_N} 
\end{figure}
\vspace{-0.3cm}

   \vspace{-0.15in}
 \begin{figure}[H]
	\centering
	\subfigure[]{\includegraphics[scale=0.3]{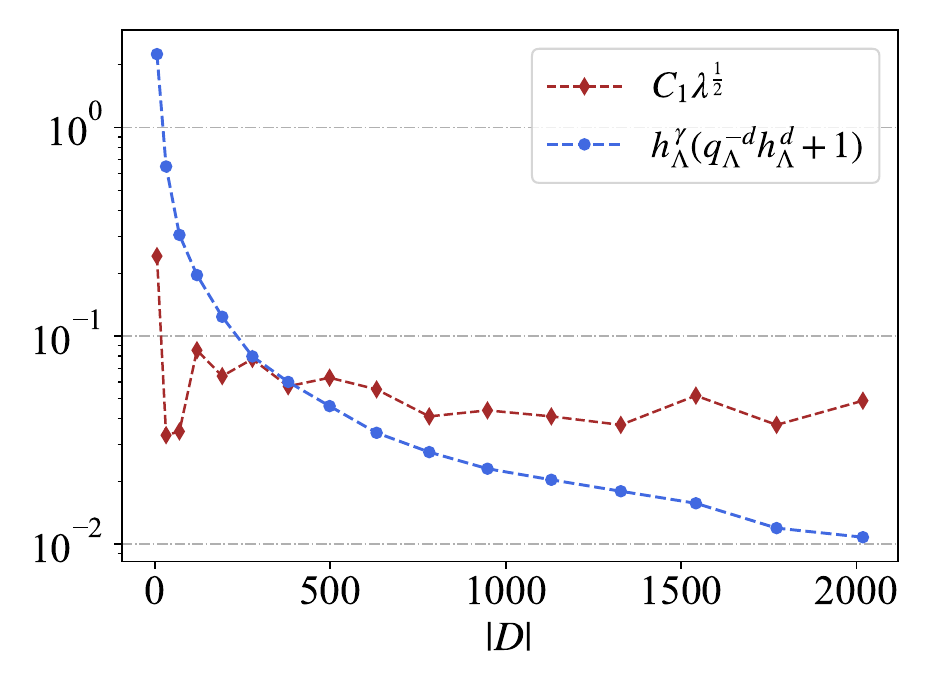}
\label{fig:distribution_bound} } 
 	\subfigure[]{\includegraphics[scale=0.3]{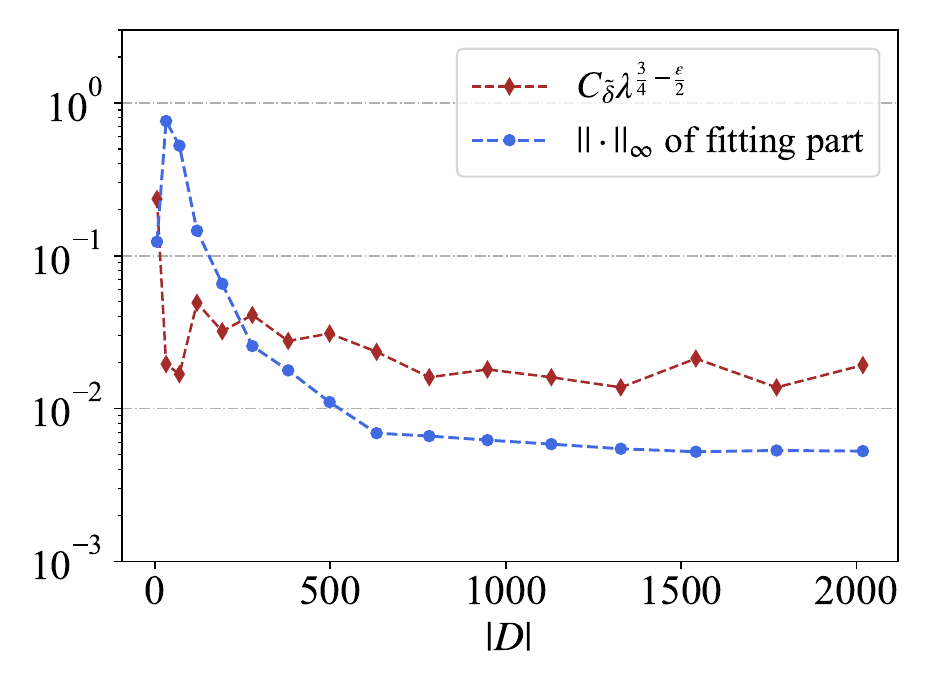}
  \label{fig:Fitting_bound}}
  	\subfigure[]{\includegraphics[scale=0.3]{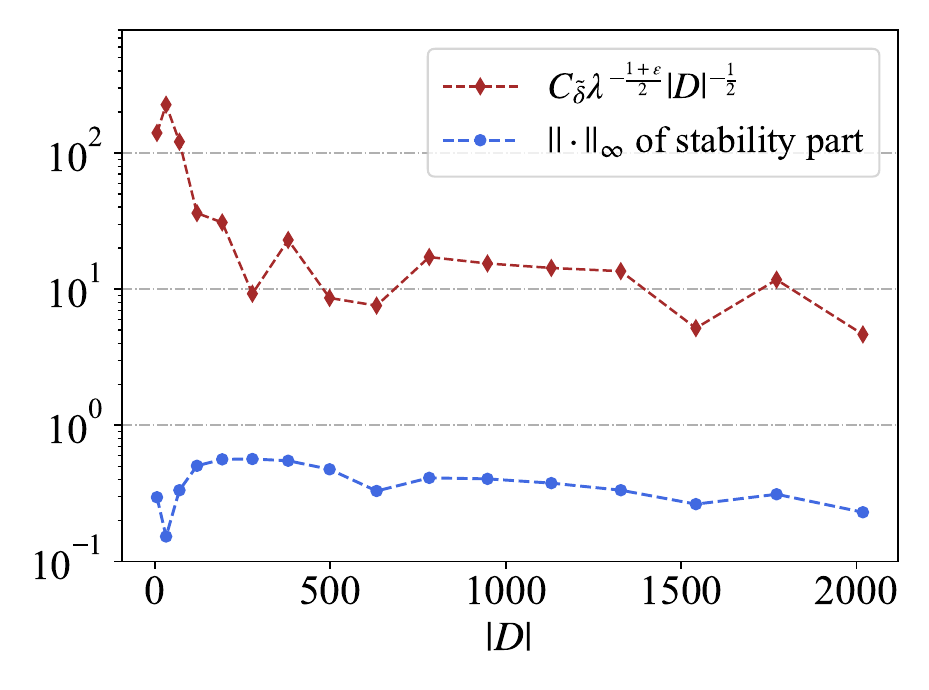}
   \label{fig:Stability_bound}} 
   \vspace{-0.15in}
	\caption{Result of Tikhonov regularization when training samples are drawn in $t$-design way with noise level of $\delta=0.3$, where $C_1=2, C_{\tilde{\delta}}=5, \epsilon=10^{-5}$. $\lambda$ here refers to the filter parameter $\mu$ of Tikhonov regularization.}
	\label{fig:infinity_condition}
	\vspace{-0.15in}
\end{figure}


\textbf{Simulation 5: Advantages of WSF under the uniform measure.} 
We also use another error measure, infinity norm $\|\cdot\|_{L^\infty(\mathbb S^d)}$ (abbreviated as $\|\cdot\|_{\infty}$ in the result figure) to evaluate the approximation performance of WSF and KI to closely connect the numerical results with theoretical analysis in \cref{Theorem:rate-spectral}. For simplicity, we only present the relevant result when training samples are drawn in $t$-design way. Taking Tikhonov regularization as an example,  \cref{fig:distribution_bound} shows that the condition \eqref{lambda-restriction} always holds for $t$-design sampling mechanism. Furthermore, for any $\epsilon>0$,   
it follows from \cref{Theorem:rate-spectral} with $d=2$, $\gamma=2$, $\alpha=2$, $M=2.5$, and $\beta=1/2+\epsilon$,
that 
\begin{eqnarray*}
    \|f_{D,\lambda}-f^*\|_{L^\infty(\mathbb S^d)}\leq   \|f_{D,\lambda}-f^*\|_{\psi}
    \leq
    C_{\tilde{\delta}}(\lambda^{-\frac{1+\epsilon}{2}}|D|^{-\frac12}+\lambda^{\frac34-\frac\epsilon2}),
\end{eqnarray*}
where 
$C_{\tilde{\delta}}$ is a constant depending only on $C_2$ and $\log\frac{3}{\tilde{\delta}}$.
 \cref{fig:Fitting_bound} and  \cref{fig:Stability_bound} show the fitting and stability part of Tikhonov regularization when noise level $\delta$ is 0.3, respectively, from which we see that \cref{Theorem:rate-spectral} does provide upper bounds on the numerical results for both parts. From  \cref{fig:infinity_MSE_N}, we find that three WSFs still outperform KI under different noise levels and $C_{\tilde{\delta}}(\lambda^{-\frac{1+\epsilon}{2}}|D|^{-\frac12}+\lambda^{\frac34-\frac\epsilon2})$ of Tikhonov does provide an upper bound of the error of Tikhonov, which further demonstrate the power of spectral filters and verify \cref{Theorem:rate-spectral}. Since the measurement conducted in the $\|\cdot\|_{L^\infty(\mathbb S^d)}$ norm is beneficial to imply  that the error is not only for a single function,  \cref{fig:infinity_MSE_N}  is crucial to show the generalizability of WSF.

\vspace{-0.15in}
 \begin{figure}[H]
	\centering
	\subfigure[$\delta=0.1 (C_{\tilde{\delta}} = 0.6)$ ]{\includegraphics[scale=0.3]{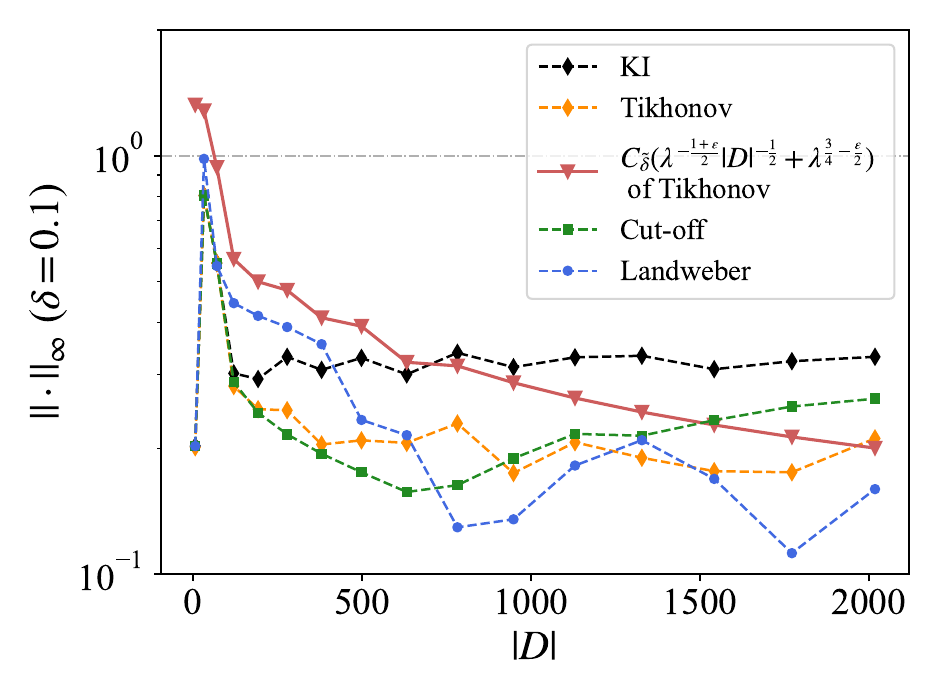}
\label{fig:infinity_noise01} } 
 	\subfigure[$\delta=0.3 (C_{\tilde{\delta}} = 1.4)$]{\includegraphics[scale=0.3]{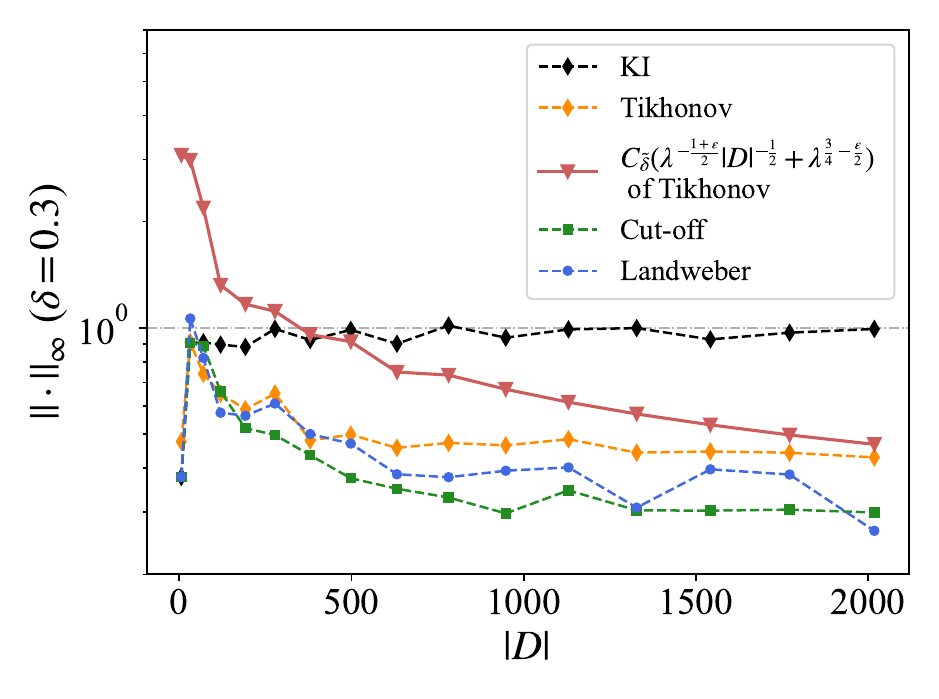}
  \label{fig:infinity_noise03}}
  	\subfigure[$\delta=0.5 (C_{\tilde{\delta}} = 2)$]{\includegraphics[scale=0.3]{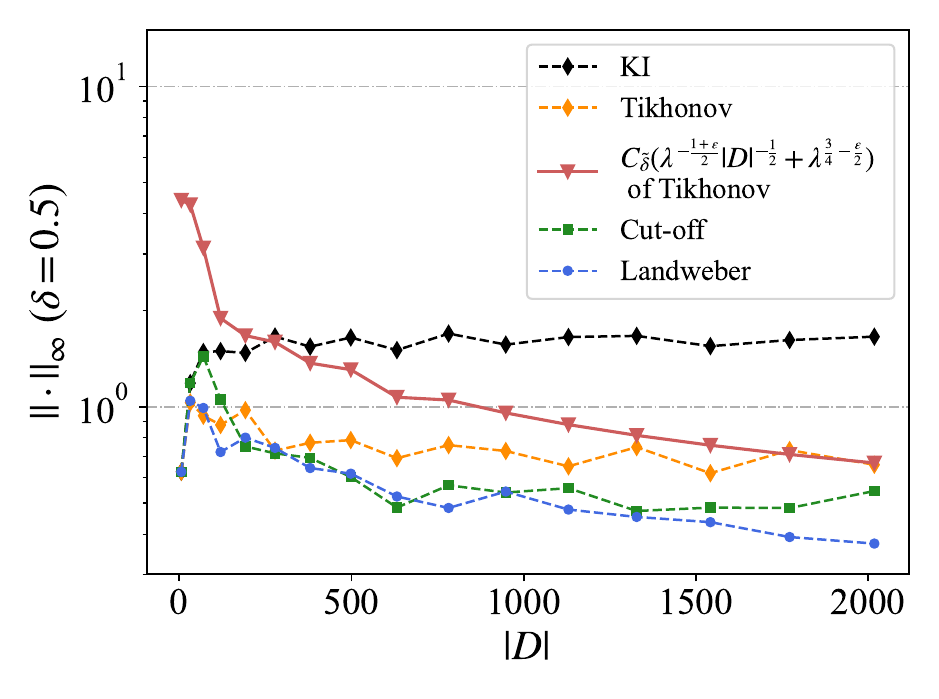}
   \label{fig:infinity_noise05}} 
   \vspace{-0.15in}
	\caption{Relations between the infinity norm  and training sample size of WSF and KI when training samples are drawn in $t$-design way with noise levles $\delta=0.1, 0.3,$ $0.5$.}
	\label{fig:infinity_MSE_N}
	\vspace{-0.1in}
\end{figure}

\vspace{-0.15in}
\begin{figure}[H]
	\centering
	\includegraphics[scale=0.4]{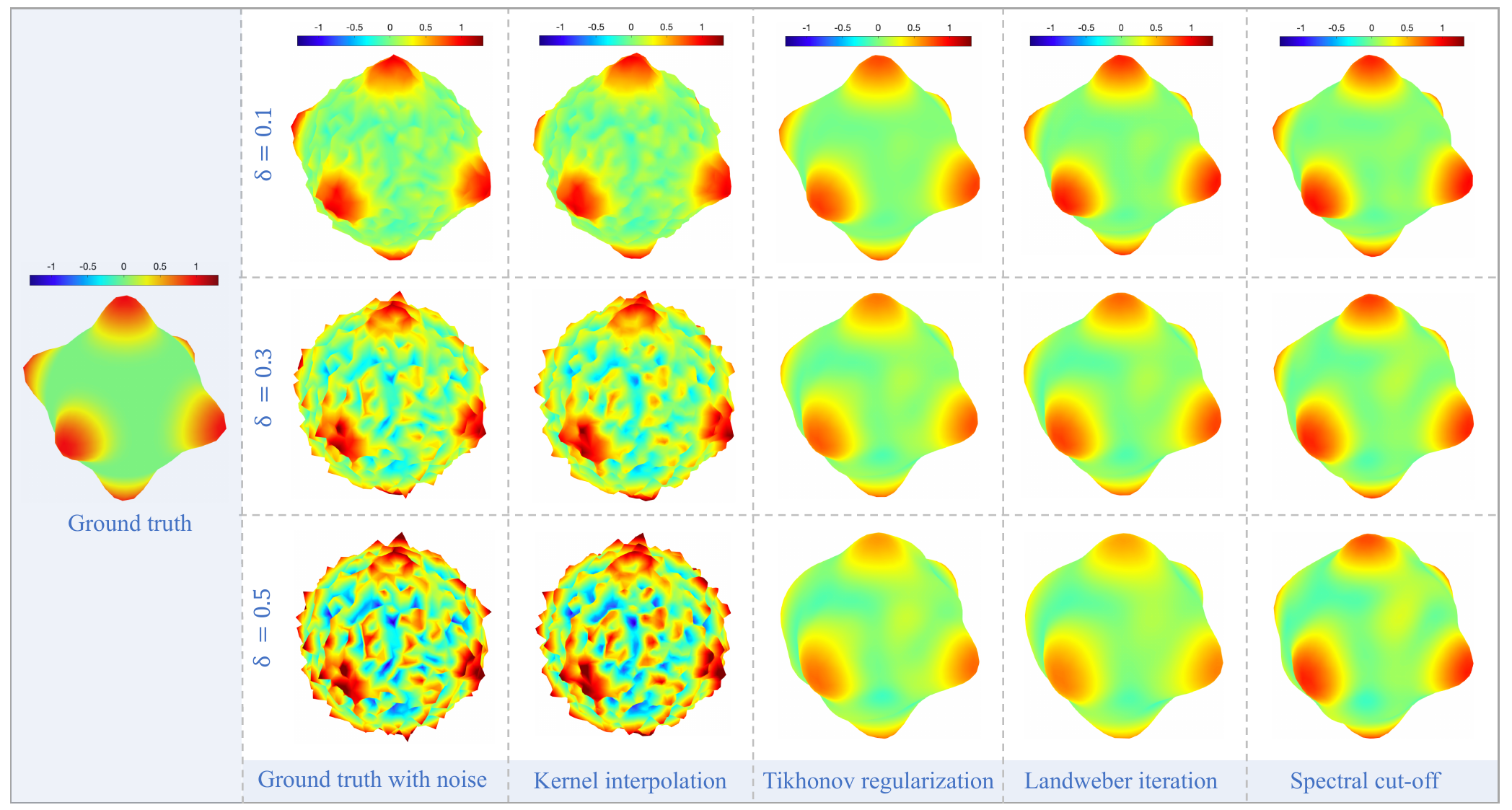}
	\caption{Visualization of the recovered results of different approaches on the 3-dimensional data under the noise levels of $\delta=0.1, 0.3$, and $0.5$, where the ground truth are training samples with $t=47$ and the other recovered results are generated based on the 6000 testing samples.}
	\label{fig: sphere_s4_tdesign_fit} 
	\vspace{-0.15in}
\end{figure}
\vspace{-0.3cm}

\textbf{Simulation 6: Visualization.} 
In this simulation, we show the intuitive comparison of the four mentioned approaches  with different noise levels. Training inputs and validation inputs  are generated in $t$-design way with $t=47$ and $t=45$, respectively, and $6000$ testing inputs are generated in the way described previously. The numerical results are shown in \cref{fig: sphere_s4_tdesign_fit}, where the figure is drawn for a single trial. 

\cref{fig: sphere_s4_tdesign_fit} shows that KI is feasible when the noise is small, but fails to predict when the noise increases. The main reason for this limitation is that KI involves all the noise in the process of interpolation, just as the first two columns show. WSF excludes the noise by  suitable spectral filters and thus is feasible for even very large noise. 
Under different noise levels, the recovered results of WSF are similar and significantly better than those of KI. These observations verify our assessments that spectral filters with appropriate filter parameters effectively circumvent the instability issue of KI without compromising its approximation performance.

\subsection{Real Data Analysis}\label{real_data}

KI and  WSF are applied to two real-world data sets.
The first, called the geomagnetic data set, contains the total intensity of magnetic field at different latitudes, longitudes, and altitudes. The second, called the wind speed data set, contains the wind speed at 2-meter at different latitudes and longitudes.
We adopt the same parameter selection strategy as in toy simulations for geomagnetic data and use five-folds cross-validation for wind speed data.  Each experiment is conducted five times for averaging.

\begin{figure}[t]
	\centering
	\includegraphics[scale=0.55]{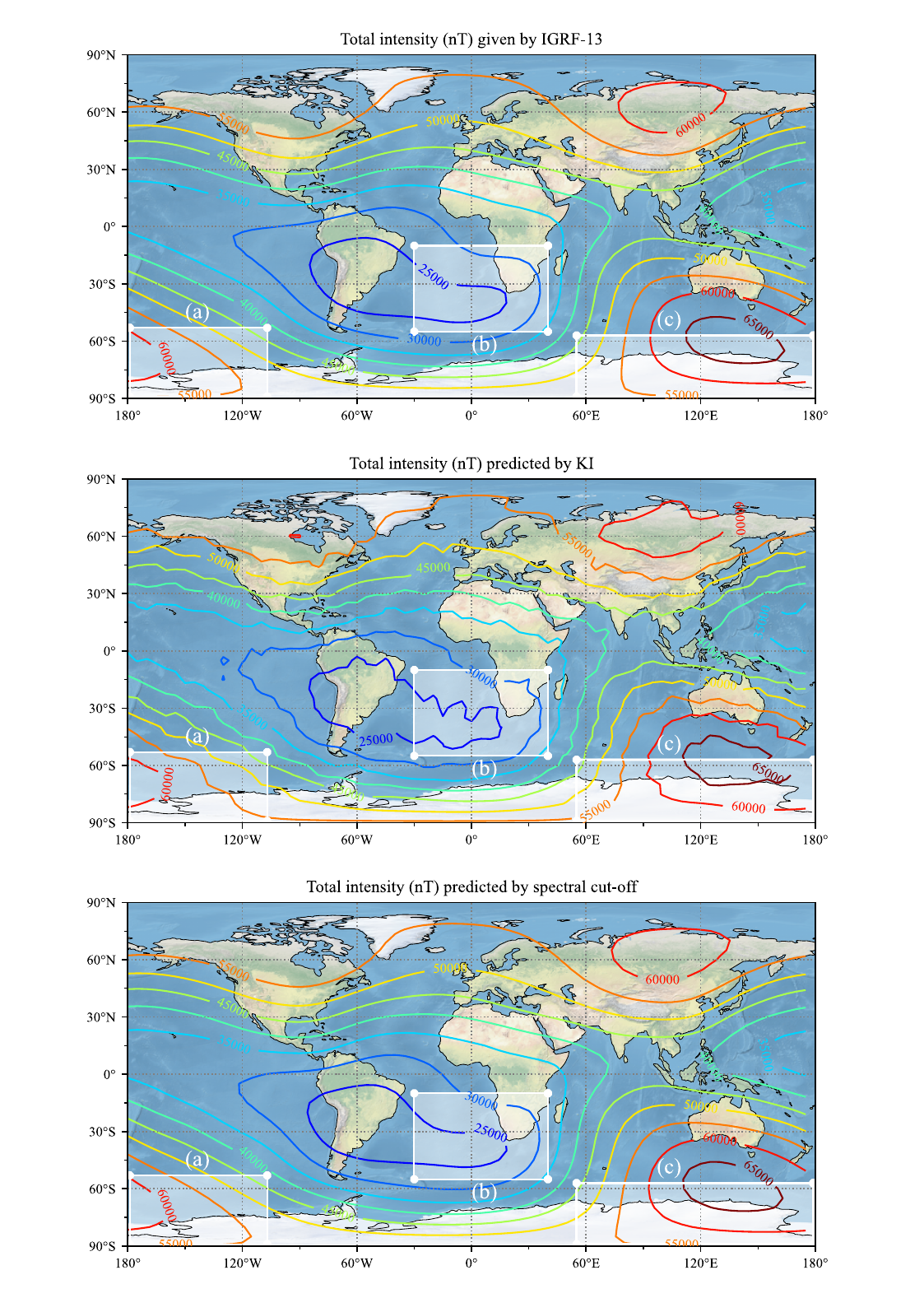}
 \vspace{-0.05in}
	\caption{Maps of total intensity given by IGRF-13 (top panel) and predicted by KI (middle panel) and spectral cut-off (bottom panel) at the WGS84 ellipsoid surface on June 6th, 2023. The projection used is the Miller cylindrical projection.}
	\label{fig: geo_fit} 
\end{figure}

{\bf Real data 1: Geomagnetic data.} 
This data set is collected from the website \url{https://geomag.bgs.ac.uk/data_service/models_compass/igrf_calc.html},
whose geomagnetic field information is provided by the IGRF-13 (the 13th generation of IGRF) \cite{alken2021international} that based on observations recorded by the real-world satellites and ground observatories. 
We only take the latitude $\phi$, longitude $\theta$, AMSL (above mean sea level, or altitude $h$), and corresponding total intensity ($nT$) on June 6th, 2023 in experiment. 



For generating training samples, we first take data on the unit sphere by Womersley's symmetric spherical $t$-design with   $t=63$. In this way,  we  take 2016
training samples by dropping the points $(0, 0, 1)$ and $(0, 0, -1)$ since the coordinate system conversion error is large at these two points. Then, we multiply each coordinate by the equatorial radius.
 Finally, we convert geocentric coordinates to geodetic coordinates (by the conversion method provided in \cite{hofmann2001gps}) and use the latter to collect total intensity. The validation samples are generated in the same way with $t$=45. To visualize geomagnetic data in Miller cylindrical projection, we collect  2664 testing samples from locations where $\phi$ and $\theta$ are sampled every five degrees and $h$ is fixed at $0km$. In order to observe the approximation performance of WSF and KI on noisy data, we synthetically add truncated Gaussian noise to the training samples with standard deviation $\delta=500$.

 \cref{fig: geo_fit} presents global maps of total intensity given by IGRF-13 \cite{alken2021international} and predicted by KI and spectral cut-off.
We put three light white rectangles $(a), (b), (c)$ in each panel to highlight their difference. Compared with the ground truth, KI distorts heavily the curves in $(a), (b), (c)$, while spectral cut-off provides predictions closer to that given by IGRF-13,   showing the necessity of introducing the spectral filters for KI.

{\bf Real data 2: Wind speed data.}
This data set is taken from the global meteorological reanalysis data set MERRA-2 developed by NASA and downloaded from the website \url{https://disc.gsfc.nasa.gov/datasets/M2T1NXSLV_5.12.4/summary}. We only take the latitude $\phi$, longitude $\theta$, and 2-meter wind speed information on May 1st, 2023 in experiment. 
We first use hourly eastward and northward wind components $u$ and $v$ to calculate the hourly wind speed $\sqrt{u^2+v^2}$, and then calculate the daily average 2-meter wind speed based on it.
This study uses  the same transformation as in \cite{li1999multiscale} used for surface air temperature, that is, denote by $[\cos\phi\cos\theta, \cos\phi\sin\theta, \sin\phi]^T$ the unit vector of points  located on earth (which is regarded as a sphere of unit radius). To visualize the wind speed data in Miller cylindrical projection, we take 2664 testing samples at locations where $\phi$ and $\theta$ are sampled every five degrees. Then,  we draw 2000 training samples from the remaining data set using simple random sampling.  In order to observe the approximation performance of WSF and KI on noisy data, we synthetically add truncated Gaussian noise to the training samples with standard deviation $\delta=0.5$. Note that the negative wind speed predictions were rectified through truncation at zero.


 \cref{fig: wind_fit} presents the maps of the daily average 2-meter wind speed given by MERRA-2 and predicted by KI and spectral cut-off.  We also put three light white rectangles $(a), (b), (c)$ in each panel to highlight their difference.  \cref{fig: wind_fit} shows that compared with KI, spectral cut-off provides predictions closer to that given by MERRA-2, which also demonstrates the excellent performance of WSF over KI.  

\vspace{-0.15in}
\begin{figure}[H]
	\centering
	\includegraphics[scale=0.5]{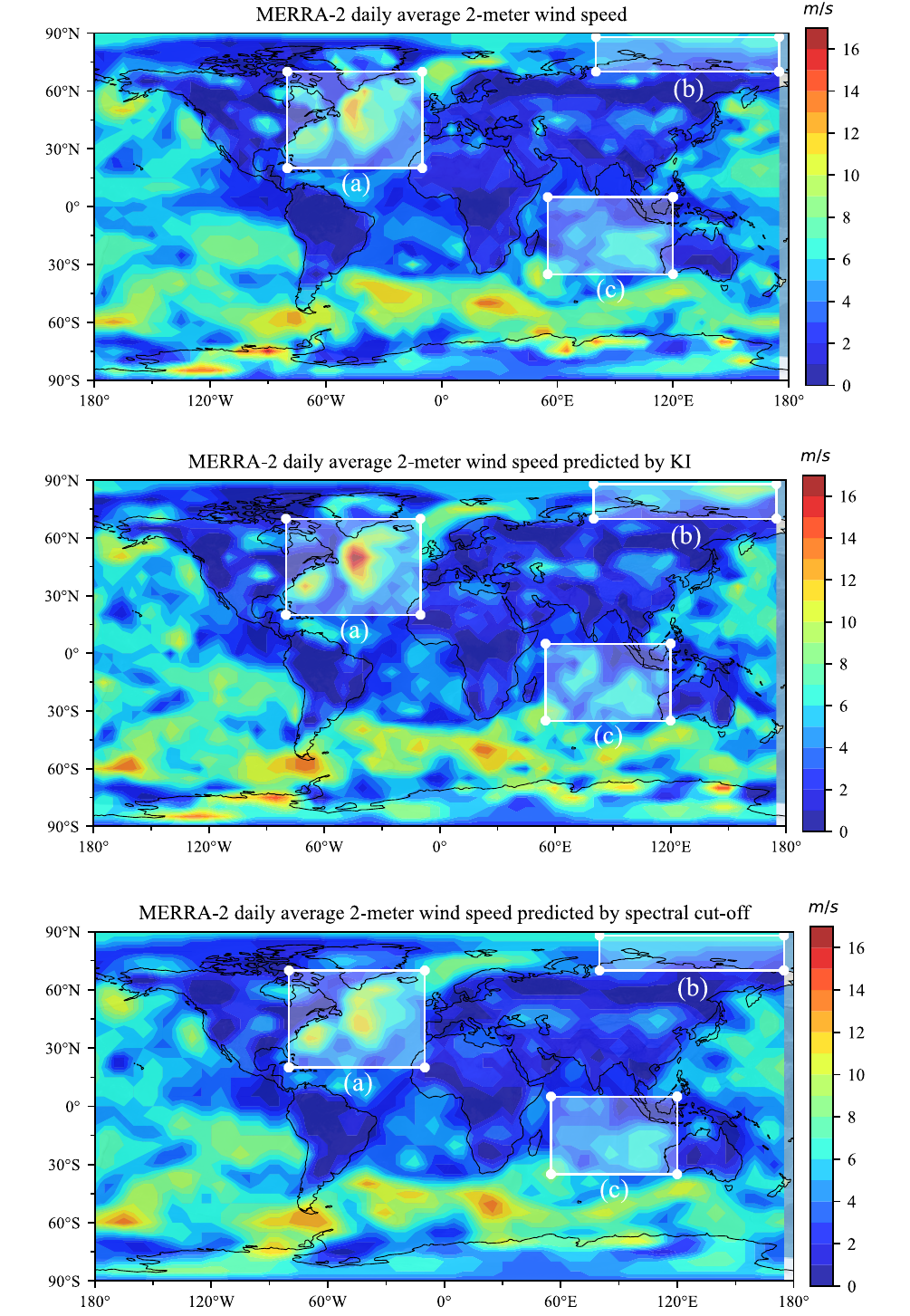}
	\caption{Maps of 2-meter wind speed drawn from  MERRA-2 data set and predicted by KI (middle panel) and spectral cut-off (bottom panel) on May 1st, 2023. The projection used is the Miller cylindrical projection.}
	\label{fig: wind_fit} 
	\vspace{-0.1in}
\end{figure}
\vspace{-0.3cm}

\section{Proofs}\label{Sec.Proofs}

In this section, we prove our theoretical results by using the recently developed integral operator approach in \cite{Feng2021radial,lin2023dis}. The integral operator approach for random samples were established in \cite{smale2004shannon,guo2017learning} by using concentration inequalities for Hilbert-valued random variables. It is obvious that such a well developed approach is not available for scattered data on spheres. Recently, \cite{Feng2021radial,lin2021subsampling,lin2023dis} established similar integral operator approach by using the positive quadrature rules on the sphere. The main novelty in our proof is to deeply study the integral operator approach so that it accommodates   WSF  in \eqref{spectral-algorithm}.  

\subsection{Operator representation for WSF}

Given a positive definite kernel $\phi$, define $\mathcal L_\phi: L^2(\mathbb S^d)\rightarrow L^2(\mathbb S^d)$ to be the integral operator  by
$$
\mathcal L_\phi f(x):=\int_{\mathbb{S}^{d}}\phi(x\cdot x')f(x')d\omega(x').
$$
It can be found in \cite{Feng2021radial} the following lemma  that  connects the regularity of $\mathcal L_\phi$ and the smoothness of the native space.

\begin{lemma}\label{Lemma:Relations-among-Native}
	Let  $\phi$,    $\psi$, $\varphi$ be three SBFs satisfying \eqref{kernel-relation} with $0\leq\beta\leq 1$ and $\alpha\geq 1$.  Then, for any $f\in\mathcal N_\psi$  there hold
	\begin{equation}\label{norm-relation}
	\|f\|_\psi =
	\|\mathcal L_\phi^{\frac{1-\beta}2}f\|_\phi,
	\end{equation}
	and for $f^*\in\mathcal N_\varphi$, there exists an $h^*\in\mathcal N_\psi$ such that
	\begin{equation}\label{source-condition}
	f^*=\mathcal L_{\phi}^{\frac{\alpha-\beta}2}h^*,\qquad  \mbox{and} \quad
	\|f^*\|_\varphi=\|h^*\|_\psi,
	\end{equation}
	where $g(\mathcal L_\phi) f$ is defined by spectral calculus:
	$
	g(\mathcal L_\phi) f= \sum_{k=0}^\infty g(\hat{\phi}_k)\sum_{\ell=1}^{Z(d, k)}\hat{f}_{k,\ell} Y_{k,\ell}(x).
	$
\end{lemma}

Let  $L_\phi:\mathcal N_\phi\rightarrow\mathcal N_\phi$  be the integral operator defined by
$
L_\phi f(x):=\int_{\mathbb{S}^{d}}\phi(x\cdot x')f(x')d\omega(x').
$ We then obtain the following lemma \cite{lin2023dis}.

\begin{lemma}\label{Lemma:integral-operator-relation}
	 For any $g:[0,\sqrt{\phi(1)}]\rightarrow\mathbb R_+$, there holds
	\begin{equation}\label{integral-operator-relation}
	g(\mathcal L_\phi)f=g(L_\phi)f,\qquad \forall\ f\in\mathcal N_\phi.
	\end{equation}
\end{lemma}




Define the weighted sampling operator $S_{D,W_{D,s}}:\mathcal N_\phi\rightarrow\mathbb R^{|D|}$ by $S_{D,W_{D,s}}f:=(\sqrt{w_{1,s}}f(x_1),\dots,\sqrt{w_{|D|,s}}f(x_{|D|}))^T$.
It is easy to check that the adjoint of $S_{D,W_{D,s}},$ denoted by $S^T_{D,W_{D,s}}$, satisfies 
$
S^T_{D,W_{D,s}}{\bf c}= \sum_{i=1}^{|D|}\sqrt{w_{i,s}}c_i\phi_{x_i}
$ for ${\bf c}:=(c_1,\dots,c_{|D|})^T.$
In this way, we get an empirical version of the positive operator $L_\phi$ as  
\begin{equation}\label{weight-empirical-operaotr}
L_{\phi,D,W_{D,s}}f=\sum_{i=1}^{|D|}w_{i,s}f(x_i)\phi_{x_i}
=S^T_{D,W_{D,s}}S_{D,W_{D,s}}f.  
\end{equation}

Since
$
     \Psi_{D,s}=W^{1/2}_{D,s}\Phi_DW^{1/2}_{D,s}=S_{D,W}S_{D,W}^T,
$
it follows from \eqref{spectral-algorithm}, \eqref{spectral-noise-free} and  the classical equality 
\cite[eqs(2.43)]{engl1996regularization}  
\begin{equation}\label{exchange-adjoint}
h(AA^T)A=Ah(A^TA) 
\end{equation}
for positive operator $A$ and any piecewise continuous function $h$ that 
\begin{equation}\label{opertor-specctral-alg}
f_{D,\lambda}=S^T_{D,W_{D,s}} g_\lambda(\Phi_{D,s})  {\bf y}_{D,W_{D,s}}= g_\lambda(L_{\phi,D,W_{D,s}}) S^T_{D,W_{D,s}} {\bf y}_{D,W_{D,s}}  
\end{equation} 
with  ${\bf y}_{D,W_{D,s}}:=(\sqrt{w_{1,s}}y_1,\dots,\sqrt{w_{|D|,s}}y_{|D|})^T$ and
\begin{equation}\label{population}
f^\diamond_{D,\lambda}=S_{D,W_{D,s}}^T g_\lambda(\Phi_{D,s})S_{D,W_{D,s}}f =g_\lambda(L_{\phi,D,W_{D,s}}) L_{\phi,D,W_{D,s}}f^*. 
\end{equation}

\subsection{Proofs of \texorpdfstring{\cref{Theorem:rate-spectral}}{Theorem:rate-spectral}}

The step-stone of the integral operator approach is to bound the approximation error by differences between the integral operator $L_\phi$ and its empirical version $L_{\phi,D,W_{D,s}}$. We introduce two important lemmas for  this purpose. 
 The first one derived in \cite{lin2023dis} quantifies the difference between $L_\phi$ and $L_{\phi,D,W_{D,s}}$.

\begin{lemma}\label{Lemma:operator-difference}
Let $\mathcal Q_{\Lambda,s}:=\{(w_{i,s},  x_i): w_{i,s}>0
\hbox{~and~}   x_i\in \Lambda\}$ be  a positive
 quadrature rule   on $\mathbb S^d$ with degree $s\in\mathbb N$.
If $\hat \phi_k\sim k^{-2\gamma}$  with $\gamma> d/2$,
then for any  $\lambda$  satisfies \eqref{lambda-restriction} and $0\leq u,v\leq 1/2$, there hold
\begin{eqnarray}\label{operaotr-difference-111}
   \|(  L_\phi+\lambda I )^{-u}(L_{\phi,D,W_{D,s}}-  L_\phi)(  L_\phi+\lambda I)^{-v}\|  
  \leq  \tilde{c}'\lambda^{-\max\{u,v\}}s^{-\gamma},
\end{eqnarray}
and
\begin{equation}\label{product-2}
	\|(L_{\phi,D,W_{D,s}}+\lambda I)^{-1/2}(L_\phi+\lambda I)^{1/2}\|\leq 2\sqrt{3}/3.
\end{equation}
\end{lemma}

The second one derived in \cite{Feng2021radial} presents an upper bound of the difference between $L_{\phi,D,W_{D,s}}f^*$ and its empirical version $S_{D,W_{D,s},s}^T{\bf y}_{D,W_{D,s}}$.

\begin{lemma}\label{lemma:value-difference-random}
	Let $\tilde{\delta}\in(0,1)$. If \eqref{Model1:fixed} holds, $\hat{\phi}_k\sim k^{-2\gamma}$ with $\gamma>d/2$ and $\{\varepsilon_i\}_{i=1}^{|D|}$ is a set of i.i.d. random variables satisfying $E[\varepsilon_i]=0$ and $|\varepsilon_i|\leq M$ for all $i=1,2,\dots,|D|$, then   with confidence $1-\tilde{\delta}$, there holds 
	\begin{eqnarray}\label{norm-difference-random}
	\|(  L_\phi+\lambda I)^{-1/2} (L_{\phi,D,W_{D,s}}f^*-S_{D,W_{D,s}}^T{\bf y}_{D,W_{D,s}})\|_\phi
	\leq 
	C_1'M\lambda^{-\frac{d}{4\gamma}} |D|^{-1/2}\log\frac{3}{\tilde{\delta}},
	\end{eqnarray}
	where $ C_1'$ is a constant  depending only on $d$, $\gamma$ and  $\|f^*\|_{\varphi}$. 
\end{lemma}

Based on the above two lemmas, we are in a position to prove \cref{Theorem:rate-spectral}.

\begin{proof}[Proof of  \cref{Theorem:rate-spectral}]
Due to \eqref{opertor-specctral-alg}, \eqref{population},  
\cref{Lemma:Relations-among-Native} and \cref{Lemma:integral-operator-relation}, we have
\begin{eqnarray*}
		&&\|f_{D,\lambda}- f^\diamond_{D,\lambda}\|_\psi
		= \left\|\mathcal L_\phi^{\frac{1-\beta}2}(f_{D,\lambda}-f^\diamond_{D,\lambda})\right\|_\phi
		=
		\left\|L_\phi^{\frac{1-\beta}2}(f_{D,\lambda}-f^\diamond_{D,\lambda})\right\|_\phi\\
  &=&
  \left\|L_\phi^{\frac{1-\beta}2}g_\lambda(L_{\phi,D,W_{D,s}})(L_{\phi,D,W_{D,s}}f^*-S_{D,W_{D,s}}^T{\bf y}_{D,W_{D,s}})\right\|_\phi\\
  &=&
  \left\|L_\phi^{\frac{1-\beta}2}(L_{\phi,D,W_{D,s}}+\lambda
		I)^{-\frac{1-\beta}2}(L_{\phi,D,W_{D,s}}+\lambda
		I)^{\frac{1-\beta}2}g_\lambda(L_{\phi,D,W_{D,s}})(L_{\phi,D,W_{D,s}}+\lambda I)^{1/2}\right.\\
  &&\left.
  (L_{\phi,D,W_{D,s}}+\lambda I)^{-1/2}( L_\phi+\lambda I)^{1/2}  (  L_\phi+\lambda I)^{-1/2}(L_{\phi,D,W_{D,s}}f^*-S_{D,W_{D,s}}^T{\bf y}_{D,W_{D,s}})\right\|_\phi.
\end{eqnarray*}
But
it can be found in \cite{Feng2021radial,lin2021subsampling} that 
	$\|L_\phi^{1/2}(L_\phi+\lambda I)^{-1/2}\|\leq
	1$, which together with $\|AB\|\leq\|A\|\|B\|$, $\|Af\|_\phi\leq\|A\|\|f\|_\phi$ for positive operator $A,B$ and $f\in\mathcal N_\phi$ that 
 \begin{eqnarray*}
		&&\|f_{D,\lambda}- f^\diamond_{D,\lambda}\|_\psi
		\leq
  \|(L_\phi+\lambda I)^{\frac{1-\beta}2}(L_{\phi,D,W_{D,s}}+\lambda
		I)^{-\frac{1-\beta}2}\|\|(L_{\phi,D,W_{D,s}}+\lambda I)^{-1/2}( L_\phi+\lambda I)^{1/2}\|\\
  &\times&\|(L_{\phi,D,W_{D,s}}+\lambda
		I)^{\frac{2-\beta}2}g_\lambda(L_{\phi,D,W_{D,s}})\|
  \|(  L_\phi+\lambda I)^{-1/2}(L_{\phi,D,W_{D,s}}f^*-S_{D,W_{D,s}}^T{\bf y}_{D,W_{D,s}})\|_\phi.
\end{eqnarray*}
Since $\frac{1-\beta}{2}\leq 1$, it follows from \cref{Lemma:operator-difference} and the well known Cordes inequality \cite{bhatia2013matrix}
\begin{equation}\label{Cordes}
	\|A^sB^s\|\le\|AB\|^s,\qquad 0<s\leq 1
\end{equation}
	for  positive operators $A,B$  that
\begin{eqnarray*}
    &&\|(L_\phi+\lambda I)^{\frac{1-\beta}2}(L_{\phi,D,W_{D,s}}+\lambda
		I)^{-\frac{1-\beta}2}\|\|(L_{\phi,D,W_{D,s}}+\lambda I)^{-1/2}( L_\phi+\lambda I)^{1/2}\|\\
  &\leq&
  \|(L_{\phi,D,W_{D,s}}+\lambda I)^{-1/2}( L_\phi+\lambda I)^{1/2}\|^{{2-\beta}}
  \leq
  \frac43.
\end{eqnarray*}
Furthermore,  
 \cref{lemma:value-difference-random}, \cref{condition1} and \eqref{product-2} yield that with confidence $1-\tilde{\delta}$ there holds
\begin{eqnarray*}
    &&\|(L_{\phi,D,W_{D,s}}+\lambda
		I)^{\frac{2-\beta}2}g_\lambda(L_{\phi,D,W_{D,s}})\|
  \|(  L_\phi+\lambda I)^{-1/2}(L_{\phi,D,W_{D,s}}f^*-S_{D,W_{D,s}}^T{\bf y}_{D,W_{D,s}})\|_\phi\\
  &\leq&
  \lambda^{-\frac{\beta}{2}}C_1'M\lambda^{-\frac{d}{4\gamma}} |D|^{-1/2} (\| L_{\phi,D,W_{D,s}} g_\lambda(L_{\phi,D,W_{D,s}})\|+\lambda\|g_\lambda(L_{\phi,D,W_{D,s}})\|)\log\frac{3}{\tilde{\delta}}\\
  &\leq&
  2b\lambda^{-\frac{\beta}{2}}C_1'M\lambda^{-\frac{d}{4\gamma}} |D|^{-1/2} \log\frac{3}{\tilde{\delta}}.
\end{eqnarray*}
Combining the above three estimates, we obtain
\begin{equation}\label{stability-111}
    \|f_{D,\lambda}- f^\diamond_{D,\lambda}\|_\psi
    \leq
    3bC_1'M\lambda^{-\frac{2\gamma\beta+d}{4\gamma}} |D|^{-1/2} \log\frac{3}{\tilde{\delta}}.
\end{equation}
We then turn to bound $\|f_{D,\lambda}^\diamond-f^*\|_\psi.$
It follows  from  \cref{Lemma:Relations-among-Native},  \cref{Lemma:integral-operator-relation},
	\eqref{product-2}, \eqref{source-condition} and  \eqref{Cordes} that
\begin{eqnarray*}
		&&\left\| f^\diamond_{D,\lambda}-f^*\right\|_\psi
		=\left\|\mathcal L_{\phi}^{\frac{1-\beta}2}(g_\lambda(L_{\phi,D,W_{D,s}})L_{\phi,D,W_{D,s}}-I)f^*\right\|_\phi\\
		&\leq&
		\left\|(L_\phi+\lambda I)^{\frac{1-\beta}2}(g_\lambda(L_{\phi,D,W_{D,s}})L_{\phi,D,W_{D,s}}
		-I)f^*\right\|_\phi\\
             &\leq&
              \|(L_\phi+\lambda I)^{\frac{1-\beta}2}(L_{\phi,D,W_{D,s}}+\lambda I)^{-\frac{1-\beta}2}\|
              \left\|(L_{\phi,D,W_{D,s}}+\lambda I)^{\frac{1-\beta}2}(g_\lambda(L_{\phi,D,W_{D,s}})L_{\phi,D,W_{D,s}}
		-I)f^*\right\|_\phi\\
		&\leq&
		(2/\sqrt{3})^{1-\beta}
		\left\|(L_{\phi,D,W_{D,s}}+\lambda I)^{\frac{1-\beta}2}(g_\lambda(L_{\phi,D,W_{D,s}})L_{\phi,D,W_{D,s}}
		-I)\mathcal L_\phi^{\frac{\alpha-\beta}{2}}h^*\right\|_\phi\\
       &\leq&
       (2/\sqrt{3})^{1-\beta}\|(L_{\phi,D,W_{D,s}}+\lambda I)^{\frac{1-\beta}2}(g_\lambda(L_{\phi,D,W_{D,s}})L_{\phi,D,W_{D,s}}
		-I)  L_\phi^{\frac{\alpha-1}2}\|
		\|\mathcal L_\phi^{\frac{1-\beta}2}h^*\|_\phi.
\end{eqnarray*}
If  $1\leq \alpha\leq3,$ \eqref{Cordes}, \eqref{condition1} and \eqref{condition2}  yield 
	\begin{eqnarray*}
		&&   
		\|(L_{\phi,D,W_{D,s}}+\lambda I)^{\frac{1-\beta}2}(g_\lambda(L_{\phi,D,W_{D,s}})L_{\phi,D,W_{D,s}}
		-I)  L_\phi^{\frac{\alpha-1}2}\|\\
        &\leq&
        \|(L_{\phi,D,W_{D,s}}+\lambda I)^{\frac{1-\beta}2}(g_\lambda(L_{\phi,D,W_{D,s}})L_{\phi,D,W_{D,s}}
		-I)(L_{\phi,D,W_{D,s}}+\lambda I)^\frac{\alpha-1}{2}\|\\
    &\times&
    \|(L_{\phi,D,W_{D,s}}+\lambda I)^{-\frac{\alpha-1}2}(L_\phi+\lambda I)^{\frac{\alpha-1}2}\|\\
		&\leq&
		(2/\sqrt{3})^{\alpha-1} \|(L_{\phi,D,W_{D,s}}+\lambda I)^{\frac{\alpha-\beta}2}(g_\lambda(L_{\phi,D,W_{D,s}})L_{\phi,D,W_{D,s}}
		-I)  \|\\
		&\leq&
		(2/\sqrt{3})^{\alpha-1}
		\|L_{\phi,D,W_{D,s}}^{\frac{\alpha-\beta}{2}}(g_\lambda(L_{\phi,D,W_{D,s}})L_{\phi,D,W_{D,s}}
		-I)\|\\
		&+&
		(2/\sqrt{3})^{\alpha-1}\lambda^{\frac{\alpha-\beta}{2}}\|g_\lambda(L_{\phi,D,W_{D,s}})L_{\phi,D,W_{D,s}}
		-I\|\\
		&\leq&
		(2/\sqrt{3})^{\alpha-1}(C_{(\alpha-\beta)/2}+C_0)\lambda^{\min\{\frac{\alpha-\beta}{2},\nu_g\}}.
	\end{eqnarray*}
For $\beta_1>0$ and $\beta_2\geq 1$, we obtain from  \eqref{condition2} and  \cref{Lemma:operator-difference} with $u=v=0$ that 
\begin{eqnarray*}
   && \|(  L_{\phi,D,W_{D,s}}+\lambda I)^{\beta_1}(g_\lambda(L_{\phi,W_s,D}) L_{\phi,W_s,D}-I)L_\phi^{\beta_2}\|\\
    &\leq&
      \|(  L_{\phi,D,W_{D,s}}+\lambda I)^{\beta_1}(g_\lambda(L_{\phi,W_s,D}) L_{\phi,W_s,D}-I) (L_\phi-L_{\phi,D,W_{D,s}})L_\phi^{\beta_2-1}\|\\
      &+&
      \|(  L_{\phi,D,W_{D,s}}+\lambda I)^{\beta_1}(g_\lambda(L_{\phi,W_s,D}) L_{\phi,W_s,D}-I)L_{\phi,D,W_{D,s}} L_\phi ^{\beta_2-1}\|\\
      &\leq&
      \|(  L_{\phi,D,W_{D,s}}+\lambda I)^{\beta_1}(g_\lambda(L_{\phi,W_s,D}) L_{\phi,W_s,D}-I)\|\|L_\phi-L_{\phi,D,W_{D,s}}\|\|L_\phi ^{\beta_2-1}\|\\
      &+&
      \|(L_{\phi,D,W_{D,s}}+\lambda I)^{\beta_1+1}(g_\lambda(L_{\phi,W_s,D})L_{\phi,W_s,D}-I)L_\phi ^{\beta_2-1}\|\\
      &\leq&
      \tilde{c}_1\lambda^{\beta_1}s^{-\gamma}+
       \|(L_{\phi,D,W_{D,s}}+\lambda I)^{\beta_1+1}(g_\lambda(L_{\phi,W_s,D})L_{\phi,W_s,D}-I)L_\phi ^{\beta_2-1}\|,
\end{eqnarray*}
where $\tilde{c}_1$ is a constant depending only on $\beta_1,\beta_2,\gamma,d$ and $\phi(1)$.
Hence, for  $\alpha>3$, we have from  \eqref{product-2} and \eqref{Cordes} that
\begin{eqnarray*}
   && \|(  L_{\phi,D,W_{D,s}}+\lambda I)^{\frac{1-\beta}2}(g_\lambda(L_{\phi,W_s,D}) L_{\phi,W_s,D}-I)L_\phi^{\frac{\alpha-1}{2}}\|\\
   &\leq&
  \tilde{c}_1\lambda^{\min\{\frac{1-\beta}2,v_g\}}s^{-\gamma}+
       \|(L_{\phi,D,W_{D,s}}+\lambda I)^{\frac{3-\beta}2}(g_\lambda(L_{\phi,W_s,D})L_{\phi,W_s,D}-I)L_\phi ^{\frac{\alpha-3}{2}}\|\\
       &\leq&
      2\tilde{c}_1s^{-\gamma} \lambda^{\min\{\frac{1-\beta}2,v_g\}}
      +
      \|(L_{\phi,D,W_{D,s}}+\lambda I)^{\frac{5-\beta}2}(g_\lambda(L_{\phi,W_s,D})L_{\phi,W_s,D}-I)L_\phi ^{\frac{\alpha-5}{2}}\|\\
      &\leq&
      \cdots\leq 
      \left[\frac{\alpha-1}{2}\right]\tilde{c}_1 s^{-\gamma} \lambda^{\min\{\frac{1-\beta}2,v_g\}} 
       + (\frac{2}{\sqrt{3}})^{\frac{\alpha-1}2-[\frac{\alpha-1}{2}]}\\
      &\times&
      \|(L_{\phi,D,W_{D,s}}+\lambda I)^{\frac{1-\beta}{2}+[\frac{\alpha-1}2]}(g_\lambda(L_{\phi,W_s,D})L_{\phi,W_s,D}-I)(L_{\phi,W_s,D}+\lambda I) ^{\frac{\alpha-1}2-[\frac{\alpha-1}{2}]}\|\\
      &\leq&
      \left[\frac{\alpha-1}{2}\right]\tilde{c}_1s^{-\gamma}\left(\lambda^{\min\{\frac{1-\beta}2,v_g\}} 
      +
      \tilde{c}_2\lambda^{\min\{\frac{\alpha-\beta}{2},v_g\}}\right)\\
      &\leq&
      \tilde{c}_3(\lambda^{\min\{\frac{1-\beta}2,v_g\}}s^{-\gamma}+\lambda^{\min\{\frac{\alpha-\beta}{2},v_g\}}),
\end{eqnarray*}
where $\tilde{c}_2,\tilde{c}_3$ are constant depending only on $\alpha$ and $\tilde{c}_1$ and $[a]$ denotes the integer part of $a\in\mathbb R$.
Combining the above three estimates, we have  
\begin{eqnarray}\label{fitting-111}
		 \left\| f^\diamond_{D,\lambda}-f^*\right\|_\psi
		\leq
		\tilde{c}_4(\lambda^{\min\{\frac{1-\beta}2,v_g\}}s^{-\gamma}\mathcal I_{\{\alpha> 3\}}+\lambda^{\min\{\frac{\alpha-\beta}{2},v_g\}}),\qquad \forall\alpha\geq 1,
\end{eqnarray}
where $\tilde{c}_4$ is a constant depending only on $\beta,\gamma$, $\phi$, $d$, $\alpha$.   Plugging \eqref{stability-111} and \eqref{fitting-111} into \eqref{error-dec-1-spectral}, we obtain that 
 for any $s_0\leq s\leq \alpha h_\Lambda^{-1}$ and $\lambda$ satisfying \eqref{lambda-restriction}, with confidence $1-\tilde{\delta}$, there holds
\begin{eqnarray*}
		\|f_{D,\lambda}-f^*\|_\psi 
		 \leq  
		 C_2(\lambda^{\min\{\frac{1-\beta}2,v_g\}}s^{-\gamma}\mathcal I_{\{\alpha> 3\}}+\lambda^{\min\{\frac{\alpha-\beta}{2},v_g\}} 
  +M\lambda^{-\frac{2\gamma\beta+d}{4\gamma}} |D|^{-1/2}\log\frac{3}{\tilde{\delta}}),
	\end{eqnarray*}	where $C_2=\max\{3bC_1',\tilde{c}_4\}$.
	This completes the proof of  \cref{Theorem:rate-spectral}.
\end{proof}

\subsection{Proofs of \texorpdfstring{\cref{Theorem:parameter}}{Theorem:parameter}}

To prove \cref{Theorem:parameter}, we need two lemmas. The first one is   a modification of  \cite[Proposition 11]{caponnetto2010cross}, we provide the proof for the sake of completeness.

\begin{lemma}\label{Lemma: cross-validation}
	Let $\{\xi_i\}_{i=1}^n$ be a sequence of real-valued independent
	random variables with mean $\mu$, satisfying $|\xi_i|\leq B$ and $
	E[(\xi_i-\mu)^2]\leq\tau^2$ for $i\in\{1,2,\dots,n\}$. Let $\{b_i\}_{i=1}^n$ be a sequence of deterministic real numbers and $b_{\max}:=\max_{i=1,\dots,n}|b_i|$. 
 For any
	$a>0$ and $\epsilon>0$, there hold
	$$
	P\left[\sum_{i=1}^nb_i(\xi_i-\mu)\geq
	a\tau^2+\epsilon\right]\leq  \exp\left\{-\frac{6a\epsilon}{3\sum_{i=1}^nb_i^2+4aBb_{\max}}\right\},
	$$
	and
	$$
	P\left[\sum_{i=1}^nb_i(\mu-\xi_i)\geq
	a\tau^2+\epsilon\right]\leq  \exp\left\{-\frac{6a\epsilon}{3\sum_{i=1}^nb_i^2+4aBb_{\max}}\right\}.
	$$
\end{lemma}

\begin{proof}
It suffices to prove the first estimate in the lemma. 
  According to  Markov's inequality   and the independence of $\{b_i\xi_i\}_{i=1}^n$, we have for any $u>0$ that 
\begin{eqnarray*}
    &&P\left[\sum_{i=1}^nb_i(\xi_i-\mu)\geq
	a\tau^2+\epsilon\right]
    =
     P\left[e^{u\sum_{i=1}^nb_i(\xi_i-\mu)}\geq e^{u( 
	a\tau^2+\epsilon)}\right]\\
 &\leq&
 e^{-u\epsilon-ua\tau^2}E\left[e^{u\sum_{i=1}^nb_i(\xi_i-\mu)}\right]
 =
  e^{-u\epsilon-ua\tau^2}\prod_{i=1}^nE\left[e^{ub_i(\xi_i-\mu)}\right].
\end{eqnarray*}
Writing $z_i=\xi_i-\mu$ and $B_1=2B$, we have for all $ub_{i}B_1\leq 3$,
\begin{eqnarray*}
   && E\left[e^{ub_iz_i}\right]
   =1+\sum_{k=1}^\infty \frac{(ub_i)^k}{k!}E[z_i^k]
   \leq
   1+0+\sum_{k=2}^\infty\frac{(ub_i)^k}{k!}B_1^{k-2}\tau^2\\
   &\leq&
   1+\frac{(ub_i)^2\tau^2}{2}\sum_{k=0}^\infty\left(\frac{B_1ub_i}{3}\right)^k
   =
   1+\frac{3u^2b_i^2\tau^2}{6-4Bub_i}
   \leq
   \exp\left\{\frac{3u^2b_{i}^2\tau^2}{6-4Bub_{i}}\right\}.
\end{eqnarray*}
Hence,
\begin{eqnarray*}
    P\left[\sum_{i=1}^nb_i(\xi_i-\mu)\geq
	a\tau^2+\epsilon\right]
   \leq 
  \exp\{-u\epsilon-ua\tau^2\}\exp\left\{\sum_{i=1}^n\frac{3u^2b_{i}^2\tau^2}{6-4Bub_{\max}}\right\}.
\end{eqnarray*}
Setting
$u=u_0=\frac{6a}{3\sum_{i=1}^nb_i^2+4aBb_{\max}}$, we obtain obviously that $u_0b_{i}B_1<3$ for any $i=1,\dots,n$. Then, 
$$
   P\left[\sum_{i=1}^nb_i(\xi_i-\mu)\geq
	a\tau^2+\epsilon\right]
   \leq 
   \exp\left\{-\frac{6a\epsilon}{3\sum_{i=1}^nb_i^2+4aBb_{\max}}\right\}.
$$
This completes the proof of \cref{Lemma: cross-validation}.
\end{proof}

The second one, derived in \cite{lin2023dis}, presents a relation between weighted average and spherical integral. 
  
\begin{lemma}\label{Lemma:quadrature-for-convolution}
Let $\hat{\phi}_k\sim k^{-2\gamma}$ with $\gamma>d/2$. If $\Lambda=\{x_i\}_{i=1}^{|\Lambda|}$ is $\theta$-quasi uniform for some $\theta>1$ and $s\leq \tilde{c}_1|\Lambda|^{1/d}$, then  for any positive quadrature  rule $\mathcal Q_{\Lambda,s}=\{(w_{i,s},  x_i): w_{i,s}>0
\hbox{~and~}   x_i\in \Lambda\}$  and
each pair of $f, g \in \mathcal N_\phi$, there holds
\begin{align*} 
    \left|\int_{\mathbb{S}^{d}}  f(x)g (x) d \omega(x)-\sum_{x_i\in\Lambda} w_{i, s} f(x_{i}) g(x_i )\right| 
    \leq 
  \tilde{c}' s^{-\gamma}
   \|f\|_\phi\|g\|_\phi,
\end{align*}
 where $\tilde{c}' $  is a constant depending only on $\gamma,\phi(1), \theta$ and $d$.  
\end{lemma}

We then prove  \cref{Theorem:parameter} as follows.

\begin{proof}[Proof of  \cref{Theorem:parameter}]
Since $f^*-f_{D^{tr},\lambda}\in \mathcal N_\phi$ for any $\lambda\in \Xi_L$, it follows from  \cref{Lemma:quadrature-for-convolution}  that
\begin{eqnarray}\label{aaa}
    &&\|f_{D^{tr},\hat{\lambda}}-f^*\|^2_{L^2(\mathbb S^d)}
    =
    \int_{\mathbb S^d}(f_{D^{tr},\hat{\lambda}}(x)-f^*(x))(f_{D^{tr},\hat{\lambda}}(x)-f^*(x))d\omega(x) \nonumber\\
    &\leq&
    \sum_{i=1}^{|D^{val}|}w_{i,s}^{val}(f_{D^{tr},\hat{\lambda}}(x_i^{val})-f^*(x_i^{val}))^2+\tilde{c}'|D|^{-\frac{\gamma}{d}}\|f^*-f_{D^{tr},\hat{\lambda},s}\|_\phi^2.
\end{eqnarray}
Due to \eqref{Model1:fixed}, we obtain
\begin{eqnarray*}
      \sum_{i=1}^{|D^{val}|}w_{i,s}^{val}(f_{D^{tr},\hat{\lambda}}(x_i^{val})- f^*(x_i^{val}))^2
    & =&\sum_{i=1}^{|D^{val}|}w_{i,s}^{val}\left((f_{D^{tr},\hat{\lambda}}(x_i^{val})-y_i^{val})^2-(y_i^{val}-f^*(x_i^{val}))^2\right) \\
  &+&
2 \sum_{i=1}^{|D^{val}|}w_{i,s}^{val}(f_{D^{tr},\hat{\lambda}}(x_i^{val})-f^*(x_i^{val}))\varepsilon_i.
\end{eqnarray*}
Let $\lambda_*=\arg\min_{\lambda\in\Xi_L}\|f_{D^{tr},\lambda}-f^*\|_{L^2(\mathbb S^d)}$. The definition of  $\hat{\lambda}$ yields
\begin{eqnarray*}
    &&
    \sum_{i=1}^{|D^{val}|}w_{i,s}^{val}\left((f_{D^{tr},\hat{\lambda}}(x_i^{val})-y_i^{val})^2-(y_i^{val}-f^*(x_i^{val}))^2\right)\\
    &\leq&
   \sum_{i=1}^{|D^{val}|}w_{i,s}^{val}\left((f_{D^{tr},{\lambda_*}}(x_i^{val})-y_i^{val})^2-(y_i^{val}-f^*(x_i^{val}))^2\right)\\
    &\leq&
     \sum_{i=1}^{|D^{val}|}w_{i,s}^{val}(f_{D^{tr}, {\lambda_*}}(x_i^{val})- f^*(x_i^{val}))^2-2 \sum_{i=1}^{|D^{val}|}w_{i,s}^{val}(f_{D^{tr}, {\lambda_*}}(x_i^{val})-f^*(x_i^{val}))\varepsilon_i.
\end{eqnarray*}
This implies
\begin{eqnarray*}
      \sum_{i=1}^{|D^{val}|}w_{i,s}^{val}(f_{D^{tr},\hat{\lambda}}(x_i^{val})- f^*(x_i^{val}))^2
  &\leq &
  \sum_{i=1}^{|D^{val}|}w_{i,s}^{val}(f_{D^{tr}, {\lambda_*}}(x_i^{val})- f^*(x_i^{val}))^2\\
  &+&
  4\max_{\lambda\in\Xi_L}\left|\sum_{i=1}^{|D^{val}|}w_{i,s}^{val}(f_{D^{tr}, {\lambda}}(x_i^{val})-f^*(x_i^{val}))\varepsilon_i\right|
\end{eqnarray*}
Inserting the above estimate into \eqref{aaa},  we obtain from  \cref{Lemma:quadrature-for-convolution} that
\begin{eqnarray*} 
      &&\|f_{D^{tr},\hat{\lambda}}-f^*\|^2_{L^2(\mathbb S^d)}
     \leq  
    \tilde{c}'|D|^{-\frac{\gamma}{d}}\|f^*-f_{D^{tr},\hat{\lambda},s}\|_\phi^2+
    \sum_{i=1}^{|D^{val}|}w_{i,s}^{val}(f_{D^{tr}, {\lambda_*}}(x_i^{val})- f^*(x_i^{val}))^2  \nonumber\\ 
  &+&
  4\max_{\lambda\in\Xi_L}\left|\sum_{i=1}^{|D^{val}|}w_{i,s}^{val}(f_{D^{tr}, {\lambda}}(x_i^{val})-f^*(x_i^{val}))\varepsilon_i\right|\nonumber \\
  &\leq&
  \|f_{D^{tr},\lambda_*}-f^*\|_{L^2(\mathbb S^d)}\\
  &+&
   2\tilde{c}'|D|^{-\frac{\gamma}{d}}\max_{\lambda\in\Xi_L}\|f^*-f_{D^{tr}, {\lambda},s}\|_\phi^2+4\max_{\lambda\in\Xi_L}\left|\sum_{i=1}^{|D^{val}|}w_{i,s}^{val}(f_{D^{tr}, {\lambda}}(x_i^{val})-f^*(x_i^{val}))\varepsilon_i\right|.
\end{eqnarray*}
Furthermore  it follows from \cref{Lemma: cross-validation} with $\xi_i=\varepsilon_i$, $b_i= w_{i,s}^{val}(f_{D^{tr}, {\lambda}}(x_i^{val})-f^*(x_i^{val})) $, $0<w_{i,s}^{val}\leq \bar{c}|D^{val}|^{-1}$,   $B=M$, $\|f_{D^{tr},\hat{\lambda}}-f^*\|_\infty\leq\sqrt{\phi (1)}\|f_{D^{tr},\hat{\lambda}}-f^*\|_K$ 
and 
$ 
    \tau^2=M^2
$ 
that 
$$
    \max_{\lambda\in\Xi_L} \left|\sum_{i=1}^{|D^{val}|}w_{i,s}^{val}(f_{D^{tr},\hat{\lambda}}(x_i^{val})-f^*(x_i^{val}))\varepsilon_i\right| \leq aM^2+\epsilon 
$$
for any $a,\epsilon>0$ with confidence at least 
$$
  1-L\exp\left\{-\frac{6a|D^{val}|\epsilon}{3\bar{c}^2 \phi (1)\max_{\lambda\in\Xi_L}\|f_{D^{tr}, {\lambda}}-f^*\|_\phi^2|D^{val}|^{-1}+4aM\bar{c}\sqrt{\phi (1)}\max_{\lambda\in\Xi_L}\|f_{D^{tr}, {\lambda}}-f^*\|_\phi}\right\}.
 $$
Setting 
$$
    \epsilon=\frac{3\bar{c}^2 \phi (1)|D^{val}|^{-1}\max_{\lambda\in\Xi_L}\|f_{D^{tr}, {\lambda}}-f^*\|_\phi^2+4aM\bar{c}\sqrt{\phi (1)}\max_{\lambda\in\Xi_L}\|f_{D^{tr}, {\lambda}}-f^*\|_\phi}{6a|D^{val}|}\log\frac{2L}{\tilde{\delta}},
$$
we get that with confidence $1-\tilde{\delta}/2$, there holds
\begin{eqnarray*}
     &&\max_{\lambda\in\Xi_L} \left|\sum_{i=1}^{|D^{val}|}w_{i,s}^{val}(f_{D^{tr},\hat{\lambda}}(x_i^{val})-f^*(x_i^{val}))\varepsilon_i\right|
     \\
     &\leq& aM^2+\frac{3\bar{c}^2 \phi (1)|D^{val}|^{-1}\max_{\lambda\in\Xi_L}\|f_{D^{tr}, {\lambda}}-f^*\|_\phi^2+4aM\bar{c}\sqrt{\phi (1)}\max_{\lambda\in\Xi_L}\|f_{D^{tr}, {\lambda}}-f^*\|_\phi}{6a|D^{val}|}\log\frac{2L}{\tilde{\delta}}.
\end{eqnarray*}
Let $a=|D^{val}|^{-1}$. We then obtain that with confidence $1-\tilde{\delta}/2$, there holds
\begin{eqnarray*}
     &&\max_{\lambda\in\Xi_L} \left|\sum_{i=1}^{|D^{val}|}w_{i,s}^{val}(f_{D^{tr},\hat{\lambda}}(x_i^{val})-f^*(x_i^{val}))\varepsilon_i\right|
     \\
     &\leq& 
     \bar{C}_1|D|^{-1}\log\frac{2L}{\tilde{\delta}}\left(\max_{\lambda\in\Xi_L}\|f_{D^{tr}, {\lambda}}-f^*\|_\phi^2+\max_{\lambda\in\Xi_L}\|f_{D^{tr}, {\lambda}}-f^*\|_\phi\right),
\end{eqnarray*}
where $\bar{C}_1=M^2+3\bar{c}^2\phi(1)+4M\bar{c}\sqrt{\phi(1)}$.
 But \cref{Theorem:rate-spectral} yields that under \eqref{lambda-restriction}, 
\begin{equation}\label{proof.RKHS-norm}
    \|f_{D^{tr},\lambda}-f^*\|_\phi 
	 \leq
	C_2  M  \lambda^{-\frac{d+2\gamma}{4\gamma}} |D^{tr}|^{-1/2}\log\frac{6}{\tilde{\delta}} 
	+
	 C_2\|f^*\|_\varphi 
		\lambda^{\frac{\alpha-1}2} 
\end{equation} 
holds with confidence $1-\tilde{\delta}/2$.  We obtain
\begin{eqnarray*}
     &&\max_{\lambda\in\Xi_L} \left|\sum_{i=1}^{|D^{val}|}w_{i,s}^{val}(f_{D^{tr},\hat{\lambda}}(x_i^{val})-f^*(x_i^{val}))\varepsilon_i\right|
     \\
     &\leq& 
     \bar{C}_2|D|^{-1}\log^3\frac{3L}{\tilde{\delta}}\max_{\lambda\in\Xi_L}\left(\lambda^{-\frac{d+2\gamma}{4\gamma}} |D^{tr}|^{-1/2} 
	+
		\lambda^{\frac{\alpha-1}2}+\lambda^{-\frac{d+2\gamma}{2\gamma}} |D^{tr}|^{-1}+ \lambda^{\alpha-1}\right)
\end{eqnarray*}
holds with confidence $1-\tilde{\delta}/2$, where $\bar{C}_2$ is a constant depending only on $M, \phi, d,\gamma$ and $\|f^*\|_\varphi.$
Therefore, with confidence $1-\tilde{\delta}$, there holds
\begin{eqnarray*} 
      &&\|f_{D^{tr},\hat{\lambda}}-f^*\|^2_{L^2(\mathbb S^d)}
     \leq  
  \|f_{D^{tr},\lambda_*}-f^*\|_{L^2(\mathbb S^d)}\\
   &+ &
   \bar{C}_3|D|^{-\frac{\gamma}{d}}\max_{\lambda\in\Xi_L}\left(\lambda^{-\frac{d+2\gamma}{2\gamma}} |D^{tr}|^{-1}+ \lambda^{\alpha-1}\right)\\
   &+&
   \bar{C}_3 |D|^{-1}\log^3\frac{3L}{\tilde{\delta}}\max_{\lambda\in\Xi_L}\left(\lambda^{-\frac{d+2\gamma}{4\gamma}} |D^{tr}|^{-1/2} 
	+
		\lambda^{\frac{\alpha-1}2}+\lambda^{-\frac{d+2\gamma}{2\gamma}} |D^{tr}|^{-1}+ \lambda^{\alpha-1}\right).
\end{eqnarray*}
	This completes the proof of  \cref{Theorem:parameter}.
\end{proof}

\section*{Acknowledgments}
The authors would like to thank AE and two anonymous reviewers for their constructive suggestions.

\bibliographystyle{siamplain}
\bibliography{references}

\end{sloppypar}

\end{document}